\documentclass{article}

\usepackage[final]{neurips_2020}

\usepackage[utf8]{inputenc} 
\usepackage[T1]{fontenc}    
\usepackage{hyperref}       
\usepackage{url}            
\usepackage{booktabs}       
\usepackage{amsfonts}       
\usepackage{nicefrac}       
\usepackage{microtype}      
\usepackage{placeins}

\usepackage{microtype}
\usepackage{graphicx}
\usepackage{subfigure}
\usepackage{booktabs} 
\usepackage{multirow}
\usepackage[table,dvipsnames]{xcolor}
\newcommand{\STAB}[1]{\begin{tabular}{@{}c@{}}#1\end{tabular}}

\usepackage{hyperref}
\usepackage{algorithmic}
\usepackage{wrapfig}

\usepackage{natbib}
\setcitestyle{round}

\usepackage[super]{nth}
\usepackage[ruled,vlined]{algorithm2e}       
\usepackage{amsfonts,amsthm,amsmath,amssymb} 
\usepackage{mathtools}
\usepackage{framed}
\usepackage{enumerate}
\usepackage{cleveref}
\usepackage{nicefrac}

\newtheorem{theorem}{Theorem}
\crefname{theorem}{theorem}{theorems}
\Crefname{theorem}{Theorem}{Theorems}

\newtheorem{assumption}{Assumption}
\crefname{assumption}{assumption}{Assumption}
\Crefname{assumption}{Assumption}{Assumptions}

\newtheorem{lemma}{Lemma}
\crefname{lemma}{lemma}{lemmas}
\Crefname{lemma}{Lemma}{Lemmas}

\newtheorem{proposition}{Proposition}
\crefname{proposition}{proposition}{propositions}
\Crefname{proposition}{Proposition}{Propositions}

\newtheorem{corollary}{Corollary}
\crefname{corollary}{corollary}{corollaries}
\Crefname{corollary}{Corollary}{Corollaries}

\crefname{conjecture}{conjecture}{conjectures}
\Crefname{conjecture}{Conjecture}{Conjectures}

\theoremstyle{definition}
\newtheorem{definition}{Definition}[section]

\newcommand{\winner}[1]{\cellcolor{gray!30}\textbf{#1}}

\newcommand{\real}[1][]{\mathbb{R}^{#1}}

\newcommand{\CVar}{\mathbb{C}}

\newcommand{\E}{\mathbb{E}}
\renewcommand{\P}{\mathbb{P}}

\newcommand{\diag}{\operatorname{diag}}

\newcommand{\kexpthree}{\textsc{k.EXP.3}\xspace}

\newcommand{\adacvar}{\textsc{Ada-CVaR}\xspace}
\newcommand{\trunkcvar}{\textsc{Trunc-CVaR}\xspace}
\newcommand{\mean}{\textsc{Mean}\xspace}
\newcommand{\softcvar}{\textsc{Soft-CVaR}\xspace}

\title{Adaptive Sampling for Stochastic Risk-Averse Learning}

\author{%
Sebastian Curi\\
Dept. of Computer Science\\
ETH Zurich\\
\texttt{scuri@inf.ethz.ch}\\
\And
Kfir Y. Levy\\
Faculty of Electrical Engineering\\
Technion\\
\texttt{kfirylevy@technion.ac.il}\\
\AND
Stefanie Jegelka\\
CSAIL\\
MIT\\
\texttt{stefje@mit.edu}\\
\And
Andreas Krause \\
Dept. of Computer Science\\
ETH Zurich\\
\texttt{krausea@inf.ethz.ch}\\
}

\begin{document}

\maketitle

\begin{abstract}
In high-stakes machine learning applications, it is crucial to not only perform well {\em on average}, but also when restricted to {\em difficult} examples.
To address this, we consider the problem of training models in a risk-averse manner.
We propose an adaptive sampling algorithm for stochastically optimizing the {\em Conditional Value-at-Risk (CVaR)} of a loss distribution, which measures its performance on the $\alpha$ fraction of most difficult examples.
We use a distributionally robust formulation of the CVaR to phrase the problem as a zero-sum game between two players, and solve it efficiently using regret minimization.
Our approach relies on sampling from structured Determinantal Point Processes (DPPs), which enables scaling it to large data sets.
Finally, we empirically demonstrate its effectiveness on large-scale convex and non-convex learning tasks. 
\end{abstract}

\section{Introduction}
Machine learning systems are increasingly deployed in high-stakes applications.
This imposes reliability requirements that are in stark discrepancy with how we currently train and evaluate these systems.
Usually, we optimize {\em expected performance} both in training and evaluation via empirical risk minimization \citep{vapnik1992principles}.
Thus, we sacrifice occasional large losses on ``difficult'' examples in order to perform well on average.
In this work, we instead consider a {\em risk-averse} optimization criterion, namely the {\em Conditional Value-at-Risk} (CVaR), also known as the Expected Shortfall.
In short, the $\alpha$-CVaR of a loss distribution is the average of the losses in the $\alpha$-tail of the distribution.

Optimizing the CVaR is well-understood in the {\em convex} setting, where duality enables a reduction to standard empirical risk minimization using a modified, truncated loss function from \citet{rockafellar2000optimization}.
Unfortunately, this approach fails when {\em stochastically} optimizing the CVaR -- especially on non-convex problems, such as training deep neural network models. 
A likely reason for this failure is that mini-batch estimates of gradients of the CVaR suffer from high variance.

\looseness -1 
To address this issue, we propose a novel {\em adaptive sampling algorithm -- \adacvar} (\Cref{sec:Adaptive}).
Our algorithm initially optimizes the mean of the losses but gradually adjusts its sampling distribution to increasingly sample tail events (difficult examples), until it eventually minimizes the CVaR (\Cref{subs:Sampler}).
Our approach naturally enables the use of standard stochastic optimizers (\Cref{subs:Learner}).
We provide convergence guarantees of the algorithm  (\Cref{subs:Convergence}) and an efficient implementation (\Cref{subs:Sampling}).
Finally, we demonstrate the performance of our algorithm in a suite of experiments (\Cref{sec:Experiments}).

\section{Related Work}

\paragraph{Risk Measures} 
Risk aversion is a well-studied human behavior, in which agents assign more weight to adverse events than to positive ones \citep{pratt1978risk}.
Approaches for modeling risk include using utility functions that emphasize larger losses \citep{rabin2013risk}; prospect theory that re-scales the probability of events \citep{kahneman2013prospect}; or direct optimization of coherent risk-measures \citep{artzner1999coherent}.
\citet{rockafellar2000optimization} introduce the CVaR as a particular instance of the latter class.
The CVaR has found many applications, such as portfolio optimization \citep{krokhmal2002portfolio} or supply chain management \citep{carneiro2010risk}, as it does not rely on specific utility or weighing functions, which offers great flexibility.

\paragraph{CVaR in Machine Learning} 
The $\nu$-SVM algorithm by \citet{scholkopf2000new} can be interpreted as optimizing the CVaR of the loss, as shown by \citet{gotoh2016cvar}.
Also related, \citet{shalev2016minimizing} propose 
to minimize the {\em maximal loss} among all samples.
The maximal loss is the limiting case of the CVaR when $\alpha \to 0$.
\citet{fan2017learning} generalize this work to the top-$k$ average loss.
Although they do not mention the relationship to the CVaR, their learning criterion is equal to the CVaR for empirical measures.
For optimization, they use 
an algorithm proposed by \citet{ogryczak2003minimizing} to optimize the maximum of the sum of $k$ functions; this algorithm is the same as the ``truncated'' algorithm of \citet{rockafellar2000optimization} to optimize the CVaR.
Recent applications of the CVaR in ML include risk-averse bandits \citep{sani2012risk}, risk-averse reinforcement learning \citep{chow2017risk}, and fairness \citep{williamson2019fairness}.
All these use the original ``truncated'' formulation of \citet{rockafellar2000optimization} to optimize the CVaR.
One of the major shortcomings of this formulation is that mini-batch gradient estimates have high variance.
In this work, we address this via a method based on adaptive sampling, inspired by \citet{shalev2016minimizing}, that allows us to handle large datasets and complex (deep neural network) models.

\paragraph{Distributionally Robust Optimization} \looseness -1
The CVaR also has a natural \emph{distributionally robust optimization} (DRO) interpretation \citep[Section 6.3]{shapiro2009lectures}, which we exploit in this paper.
For example, \citet{ahmadi2012entropic} introduces the entropic value-at-risk by considering a different DRO set.
\citet{duchi2016statistics,namkoong2017variance,esfahani2018data,kirschner20} address related DRO problems, but with different uncertainty sets.
We use the DRO formulation of the CVaR to phrase its optimization as a game.
To solve the game, we propose an adaptive algorithm for the learning problem.
Our algorithm is most related to \citep{namkoong2016stochastic}, who develop an algorithm for DRO sets induced by Cressie-Read $f$-divergences. 
Instead, we use a different DRO set that arises in common data sets \citep{mehrabi2019survey} and we provide an \emph{efficient} algorithm to solve the DRO problem in large-scale datasets.

\section{Problem Statement}
\looseness -1 We consider supervised learning with a {\em risk-averse learner}.
The learner has a data set comprised of i.i.d.~samples from an unknown distribution, i.e., $D = \left\{ (x_1, y_1), \ldots (x_N, y_N) \right\} \in (\mathcal{X} \times \mathcal{Y})^N \sim \mathcal{D}^N$, and her goal is to learn a function $h_\theta: \mathcal{X} \to \mathcal{R}$ that is parametrized by $\theta \in \Theta \subset \real[d]$.
The performance of $h_\theta$ at a data point is measured by a {\em loss function} $l: \Theta\times\mathcal{X} \times \mathcal{Y} \to [0, 1]$.
We write the random variable $L_i(\theta) = l(\theta; x_i, y_i)$.
The learner's goal is to minimize the {\em CVaR of the losses} on the (unknown) distribution $\mathcal{D}$  w.r.t.~the parameters $\theta$.

\paragraph{CVaR properties} 
\begin{wrapfigure}{r}{0.45\textwidth}
\centering
        \includegraphics[width=0.45\textwidth]{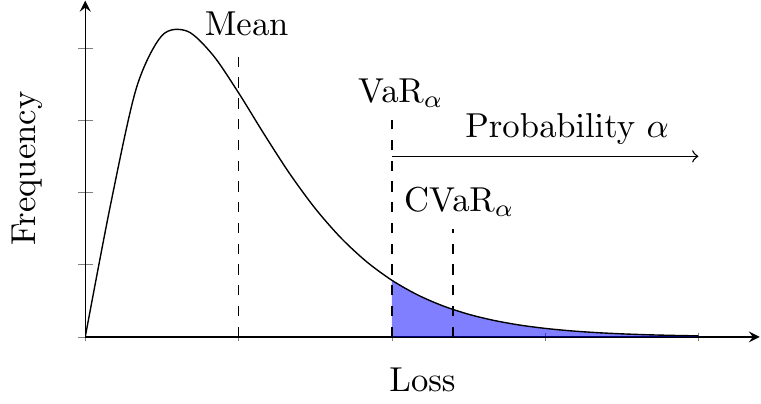}
    \caption{CVaR of a loss distribution}
    \label{fig:CVaR}
\end{wrapfigure} 
The CVaR of a random variable $L \sim P$ is defined as $\CVar^\alpha[L] = \E_P[L | L \geq \ell^\alpha]$, where $\ell^\alpha$ is the $1-\alpha$ quantile of the distribution, also called the {\em Value-at-Risk} (VaR).
We illustrate the mean, VaR and CVaR of a typical loss distribution in \Cref{fig:CVaR}.
It can be shown that the CVaR of a random variable has a natural {\em distributionally robust optimization} (DRO) formulation, namely as the expected value of the same random variable under a {\em different} law.
This law arises from the following optimization problem \citep[Sec.~6.3]{shapiro2009lectures}:
\begin{equation}
    \CVar^\alpha[L] = \max_{Q \in \mathcal{Q}^\alpha} \E_{Q}[L], \label{eq:DRO}
\end{equation}
where $\mathcal{Q}^\alpha = \left\{Q \ll P, \frac{dQ}{dP} \leq \frac{1}{\alpha} \right\} $. Here, $Q\ll P$ means that $Q$ is absolutely continuous w.r.t.~$P$.
The distribution $Q^\star$ that solves Problem \eqref{eq:DRO} places all the mass uniformly in the tail, i.e., the blue shaded region of \Cref{fig:CVaR}. Thus, optimizing the CVaR can be viewed as {\em guarding against a particular kind of distribution shift}, which reweighs arbitrary parts of the data up to a certain amount $\frac{1}{\alpha}$.
\citet{rockafellar2000optimization} prove strong duality for Problem \eqref{eq:DRO}.
The dual objective is:
\begin{equation}
    \CVar^\alpha[L] = \min_{\ell \in \real} \ell + \frac{1}{\alpha} \E_{P} \left[ \max\left\{0, L -\ell \right\} \right].\label{eq:Rockafellar}
\end{equation}

\paragraph{Learning with the CVaR}
Problem~\eqref{eq:Rockafellar} can be used to estimate the CVaR of a random variable by replacing the expectation $\E_{P}$ by the empirical mean $\hat{\E}$, yielding
\begin{equation}
   \min_{\ell \in \real, \theta \in \Theta} \ell + \frac{1}{\alpha N} 
   \sum\nolimits_{i=1}^N
   \left[ \max\left\{0, L_i(\theta)-\ell \right\} \right].\label{eq:BadLearning}
\end{equation}
For convex $L_i$,  Problem~\eqref{eq:BadLearning} has computable subgradients, and hence lends itself to subgradient-based optimization.
Furthermore, 
in this case Problem~\eqref{eq:BadLearning} is jointly convex in $(\ell, \theta)$. We refer to this standard approach as \trunkcvar, as it effectively optimizes a modified loss, truncated at $\ell$.

Problem~\eqref{eq:BadLearning} is indeed a sensible learning objective in the sense that the empirical CVaR concentrates around the population CVaR uniformly for all functions $h \in \mathcal{H}$.
%
\begin{proposition} \label{prop:uniform-convergence}
    Let $h: \mathcal{X} \to \mathcal{Y}$ be a finite function class $|\mathcal{H}|$.
    Let $L(h): \mathcal{H} \to [0, 1]$ be a random variable.
    Then, for any $0 < \alpha \leq 1$, with probability at least $1-\delta$, 
    \begin{align*} 
        \E\left[ \sup_{h \in \mathcal{H}} \left| \widehat{\CVar}^\alpha [L(h)] - \CVar^\alpha[L(h)] \right| \right] \leq \frac{1}{\alpha}\sqrt{\frac{\log(2 |\mathcal{H}| /\delta)}{N}}.
    \end{align*}
\end{proposition}
\begin{proof}See \Cref{proof:uniform-convergence}.
\end{proof}
\vspace{-1em}
The result above is easily extended to classes $\mathcal{H}$ with finite VC (pseudo-)dimension. 
Concurrent to this work, \citet{lee2020learning}, \citet{soma2020statistical}, and \citet{mhammedi2020pac} present similar statistical rates based on different assumptions.

\paragraph{Challenges for Stochastic Optimization} 
In the common case that a variant of SGD is used to optimize the learning problem~\eqref{eq:BadLearning}, the expectation is approximated with a mini-batch of data.
But, when this batch is sampled uniformly at random from the data, only a fraction $\alpha$ of points will contain gradient information.
The gradient of the remaining points gets truncated to zero by the $\max\{\cdot\}$ non-linearity.
Furthermore, the gradient of the examples that \emph{do} contain information is scaled by $1/\alpha$, leading to exploding gradients.
These facts make stochastic optimization of Problem~\eqref{eq:BadLearning} extremely challenging, as we demonstrate in \Cref{sec:Experiments}.

\looseness -1 Our key observation is that the root of the problem lies in the {\em mismatch} between the sampling distribution $P$ and the unknown distribution $Q^\star$, from which we would ideally like to sample.
In fact, Problem~\eqref{eq:BadLearning} can be interpreted as a form of rejection sampling -- samples with losses smaller than $\ell$ are rejected.
It is well known that Monte Carlo estimation of rare events suffers from high variance \citep{rubino2009rare}.
To address this issue, we propose a novel sampling algorithm that \emph{adaptively} learns to sample events from the distribution $Q^\star$ while optimizing the model parameters $\theta$.

\section{\adacvar: Adaptive Sampling for CVaR Optimization} \label{sec:Adaptive}

\looseness -1
We directly address the DRO problem~\eqref{eq:DRO} on the empirical measure $\hat{P}$ for learning.
The DRO set is  $\mathcal{Q}^\alpha = \left\{q \in \real[N] \mid 0 \leq q_i \leq \frac{1}{k}, \sum_i q_i = 1 \right \}$ with $k=\lfloor \alpha N \rfloor$.
The learning problem becomes:
\begin{equation}
  \min_{\theta \in \Theta} \max_{q \in \mathcal{Q}^\alpha} \E_{q}[L_i(\theta)] = \min_{\theta \in \Theta} \max_{q \in \mathcal{Q}^\alpha} q^\top L(\theta), \label{eq:DROLearning}
\end{equation}  
where $L(\theta) \in \real[N]$ has $i$-th index $L_i(\theta)$.
The learning problem~\eqref{eq:DROLearning} is a minimax game between a $\theta$-player (the learner), whose goal is to \emph{minimize} the objective function by selecting $\theta \in \Theta$, against a $q$-player (the sampler), whose goal is to \emph{maximize} the objective function by selecting $q \in \mathcal{Q}^\alpha$.

To solve the game \eqref{eq:DROLearning}, we use techniques from regret minimization with partial (bandit) feedback. 
In particular, we exploit that one can solve minimax games by viewing both players as online learners that compete, and by equipping each of them with no-regret algorithms \citep{freund1999adaptive}.
With the partial (bandit) feedback model, we only need to consider a {\em small} subset of the data in each iteration of the optimization algorithm. 
In contrast, full-information feedback would require a full pass over the data per iteration, invalidating all benefits of stochastic optimization.

\begin{wrapfigure}{r}{0.5\textwidth}
\begin{minipage}{0.5\textwidth}
\vspace{-1.5em}
\begin{algorithm}[H]
 \caption{\adacvar} \label{alg:Game}
\begin{algorithmic}[1]
    \INPUT Learning rates $\eta_s$, $\eta_l$. 
    \STATE \textbf{Sampler:} Initialize k-DPP $w_1 = \mathbf{1}_N$.
    \STATE \textbf{Learner:} Initialize parameters $\theta_0 \in \Theta$.
    \FOR{$t=1,\ldots, T$}
        \STATE \textbf{Sampler:} Sample data point $i_t \sim q_{t} = \frac{1}{k}\P_{w_t}(i)$.
        \STATE \textbf{Learner:} $\theta_t = \theta_{t-1} - \eta_l \nabla L_{i_t}(\theta_{t-1}) $.
        \STATE \textbf{Sampler}: Build estimate $\hat{L}_{t} = \frac{L_{ i_t}(\theta_{t})}{q_{t, i_t}} [[i == i_t]]$.
        \STATE \textbf{Sampler}: Update k-DPP $w_{t+1, i} = w_{t, i} e^{\eta_s \hat{L}_{t, i}}$.
    \ENDFOR
    \OUTPUT $\bar{\theta}, \bar{q} \sim_{u.a.r} \{(\theta_t, q_t)\}_{t=1}^T$
\end{algorithmic}
\end{algorithm}
\end{minipage}
\vspace{-2em}
\end{wrapfigure}
Next, we describe and analyze an online learning algorithm for each of the two players and prove guarantees with respect to the DRO problem~\eqref{eq:DROLearning}. 
We outline the final algorithm, which we call \adacvar, in \Cref{alg:Game}, where we use an adaptive sampling scheme for the $q$-player and SGD for the $\theta$-player.
Initially, the $q$-player (sampler) plays the uniform distribution and the $\theta$-player (learner) selects any parameter in the set $\Theta$. In iteration $t$, the sampler samples a data point (or a mini-batch) with respect to the distribution $q_t$. Then, the learner performs an SGD step on the sample(s) selected by the sampler player\footnote{Note that we do \emph{not} use any importance sampling correction.}. 
Finally, the $q$-player adapts the distribution to favor examples with higher loss and thus maximize the objective in \eqref{eq:DROLearning}.

\subsection{Sampler (\texorpdfstring{$q$}{}-Player) Algorithm} \label{subs:Sampler}
In every iteration $t$, the learner player sets a vector of losses through $\theta_t$.
We denote by $L(\theta_t; x_i, y_i) = L_{t,i}$ the loss at time $t$ for example $i$ and by $L(\theta_t) = L_t$ the vector of losses.
The sampler player chooses an index $i_t$ (or a mini-batch) and a vector $q_t$. Then, only $L_{t, i_t}$ is revealed and she suffers a cost $q_t^\top L_t$.
In such setting, the best the player can aim to do is to minimize its \emph{regret}:
\begin{equation}
    \operatorname{SR}_T \coloneqq  \max_{q \in \mathcal{Q}^\alpha} \sum\nolimits_{t=1}^T q^\top L_t - \sum\nolimits_{t=1}^T q_t^\top L_t. \label{eq:SamplerRegret}
\end{equation}
The regret measures how good the sequence of actions of the sampler is, compared to the best single action (i.e., distribution over the data) in hindsight, after seeing the sequence of iterates $L_t$.
The sampler player problem is a linear adversarial bandit. Exploration and sampling in this setting are hard \citep{bubeck2012towards}. Our efficient implementation exploits the specific combinatorial structure.

\looseness -1 In particular, the DRO set $\mathcal{Q}^\alpha$ is a polytope with $\binom{N}{k}$ vertices, each corresponding to a different subset $I$ of size $k$ of the ground set $2^{[N]}$.
As the inner optimization problem over $q$ in \eqref{eq:DROLearning} is a linear program, the optimal solution $q^\star$ is a vertex. 
Thus, the sampler problem can be reduced to a {\em best subset selection} problem: find the best set among all size-$k$ subsets $\mathcal{I}_k = \left\{ I \subseteq 2^{[N]} \mid |I| = k \right\}$. 
Here, the value of a set $I$ at time $t$ is simply the average of the losses $(1/k)\sum_{i\in I}L_i(\theta_t)$.
The problem of maximizing the value over time $t$ can be viewed as a {\em combinatorial bandit} problem, as we have a combinatorial set of ``arms'', one per $I \in \mathcal{I}_k$ \citep[Chapter 30]{lattimore2018bandit}. 
Building on \citet{alatur2020multi}, we develop an efficient algorithm for the sampler. 


\looseness -1 \paragraph{Starting Point: EXP.3} A well known no-regret bandit algorithm is the celebrated EXP.3 algorithm \citep{auer2002nonstochastic}. 
Let $\Delta_{I} \coloneqq \left\{\tilde{W} \in \real[\binom{N}{k}] | \sum_I \tilde{W}_I = 1, \tilde{W}_I \geq 0 \right\}$ be the simplex of distributions over the $\binom{N}{k}$ subsets.
Finding the best distribution $W_I^\star \in \Delta_{I}$ is equivalent to finding the best subset $I^\star \in \mathcal{I}_k$. By transitivity, this is equivalent to finding the best $q^\star \in \mathcal{Q}^\alpha$.
To do this, EXP.3 maintains a vector $W_{I, t} \in \Delta_{I}$, samples an element $I_t \sim W_{I, t}$ and observes a loss associated with element $I_t$.
Finally, it updates the distribution using multiplicative weights. Unfortunately, EXP.3 is {\em intractable in two ways}: Sampling a $k$-subset $I_t$ would require evaluating the losses of $k=\lfloor \alpha N\rfloor$ data points, which is impractical. Furthermore, the naive EXP.3 algorithm is intractable because the dimension of $W_{I,t}$ is {\em exponential} in $k$. 
In turn, the regret of this algorithm also depends on the dimension of $W_{I,t}$.

\paragraph{Efficiency through Structure} The crucial insight is that we can {\em exploit the combinatorial structure of the problem and additivity of the loss} to exponentially improve efficiency.
First, we exploit that weights of individual elements and sets of them are related by $W_{t,I} = \sum_{i\in I}w_{t,i}$.  Thus, instead of observing the loss $L_{I_t}$, we let the $q$-player sample only a {\em single element} $i_t$ uniformly at random from the set $I_t \sim W_{I, t}$, observe its loss $L_{i_t}$, and use it to update a weight vector $w_{t,i}$.
The single element $i_t$ sampled by the algorithm provides information about the loss of all $\binom{N-1}{k-1}$ {\em sets} that contain $i_t$. This allows us to obtain regret guarantees that are sub-linear in $N$ (rather than in $N^k$). 
Second, we {\em exponentially improve computational cost} by developing an algorithm that maintains a vector $w \in \real[N]$ and uses \emph{k-Determinantal Point Processes} to map it to distributions over subsets of size $k$. 

\begin{definition}[k-DPP, \citet{kulesza2012determinantal}] \looseness -1 A $k$-Determinantal Point Process over a ground set $N$ is a distribution over all subsets of size $k$ s.t.~the probability of a set is $\P(I) \propto \operatorname{det}(K_I)$,
\looseness -1 where $K$ is a positive definite kernel matrix and $K_I$ is the submatrix of $K$ indexed by $I$.\hfill $\blacksquare$
\end{definition}

In particular, we consider k-DPPs with {\em diagonal} kernel matrices $K = \diag{w}$, with $w \in \real[N]_{\geq 0}$ and at least $k$ strictly positive elements. 
This family of distributions is sufficient to contain, for example, the uniform distribution over the $\binom{N}{k}$ subsets and all the vertices of $\mathcal{Q}^\alpha$. 
We use such k-DPPs to {\em efficiently} map a vector of size $N$ to a distribution over $\binom{N}{k}$ subsets.
We also denote the marginal probability of element $i$ by $\P_w(i)$.
It is easy to verify that the vector of marginals $\frac{1}{k}\P_w(i) \in \mathcal{Q}^\alpha$. 
Hence, we directly use the k-DPP marginals as the sampler's decision variables.


We can finally describe the sampler algorithm.
We initialize the k-DPP kernel with the uniform distribution $w_1 = \mathbf{1}_N$. 
In iteration $t$, the sampler plays the distribution $q_t = \frac{1}{k}\P_{w_t}(\cdot) \in \mathcal{Q}^\alpha$ and samples an element $i_t \sim q_t$.
The loss at index $i_t$, $L_{t, i_t}$, is revealed to the sampler and only the index $i_t$ of $w_t$ is updated according to the multiplicative update $w_{t+1, i_t} = w_{t+1, i_t}e^{k L_{t, i_t} / q_{t, i_t}}$. 


This approach addresses the disadvantages of the EXP.3 algorithm. 
Computationally, it only requires $O(N)$ memory.
After sampling every element $i_t$, the distribution over the $\binom{N-1}{k-1}$ sets that contain $i_t$ are updated. This yield rates that depend sub-linearly on the data set size which we prove next.

\begin{lemma} \label{lemma:Sampler}
    Let the sampler player play the \adacvar Algorithm with $\eta_s = \sqrt{\frac{\log N}{N T}}$.
    Then, for any sequence of losses she suffers a regret \eqref{eq:SamplerRegret} of at most $O(\sqrt{TN\log N})$.
\end{lemma}

\begin{proof}[Proof sketch]
\looseness -1 For a detailed proof please refer to \Cref{proof:Sampler}.
First, we prove in \Cref{prop:KDPP_DRO} that the iterates of the algorithm are effectively in $\mathcal{Q}^\alpha$.
Next, we prove in \Cref{prop:comparator} that the comparator in the regret of \citet{alatur2020multi} and in the sampler regret \eqref{eq:SamplerRegret} have the same value (scaled by $k$). 
Finally, the result follows as a corollary from these propositions and \citet[Lemma 1]{alatur2020multi}.
\end{proof}

\subsection{Learner (\texorpdfstring{$\theta$-}{}Player) Algorithm} \label{subs:Learner}
Analogous to the sampler player, the learner player seeks to minimize its regret
\begin{equation}
    \operatorname{LR}_T \coloneqq  \sum\nolimits_{t=1}^T q_t^\top L(\theta_t) - \min_{\theta \in \Theta} \sum\nolimits_{t=1}^T q_t^\top L(\theta).\label{eq:LearnerRegret}
\end{equation}
Crucially, the learner can choose $\theta_t$ {\em after} the sampler selects $q_t$.
Thus, the learner can play the Be-The-Leader (BTL) algorithm:
\begin{equation}
    \theta_t = \arg\min_{\theta \in \Theta} \sum\nolimits_{\tau=1}^t q_\tau^\top L(\theta) = \arg\min_{\theta \in \Theta} \bar{q}_t^\top L(\theta), \label{eq:BTL}
\end{equation}
where $\bar{q}_t = \frac{1}{t} \sum_{\tau=1}^t q_\tau$ is the average distribution (up to time $t$) that the sampler player proposes. 

Instead of assuming access to an exact optimizer, we assume to have an ERM oracle available. 
\begin{assumption}[$\epsilon_{\mathrm{oracle}}$-correct ERM Oracle]     \label{assumption:ERM}
    The learner has access to an ERM oracle that takes a distribution $q$ over the dataset as input and outputs $\hat{\theta}$, such that
    \begin{align*}
        q^\top L(\hat{\theta}) \leq \min_{\theta \in \Theta } q^\top L(\theta) + \epsilon_{\mathrm{oracle}}.
    \end{align*} 
\end{assumption}
To implement the ERM Oracle, the learner player must  solve a \emph{weighted} empirical loss minimization in Problem~\eqref{eq:BTL} in every round.
For non-convex problems, this is in general NP-hard \citep{murty1987some}, so
obtaining efficient and provably no-regret guarantees in the non-convex setting seems unrealistic in general.

Despite this hardness, the success of deep learning empirically demonstrates that stochastic optimization algorithms such as SGD are able to find very good (even if not necessarily optimal) solutions for the ERM-like non-convex problems.
Furthermore, SGD on the sequence of samples $\left\{i_\tau \sim q_\tau\right\}_{\tau=1}^t$ approximately solves the BTL problem. 
To see why, we note that such sequence of samples is an unbiased estimator of $\bar{q}_t$ from the BTL algorithm~\eqref{eq:BTL}. 
Then, for the freshly sampled $i_t \sim q_t$, a learner that chooses $\theta_{t} \coloneqq \theta_{t-1} - \eta_l \nabla L_{i_t}(\theta_{t-1})$ is (approximately) solving the BTL algorithm with SGD. 

\begin{lemma} \label{lemma:Learner}
A learner player that plays the BTL algorithm with access to an ERM oracle as in \Cref{assumption:ERM}, achieves at most $\epsilon_{\mathrm{oracle}}T$ regret.
\end{lemma}

\begin{proof}
    See \Cref{proof:Learner}.
\end{proof}

For convex problems, we know that it is not necessary to solve the BTL problem~\eqref{eq:BTL}, and algorithms such as online projected gradient descent \citep{zinkevich2003online} achieve no-regret guarantees. As shown in \Cref{app:ConvexLearner}, the learner suffers $\operatorname{LR}_T = O(\sqrt{T})$ regret by playing SGD in convex problems.
\subsection{Guarantees for CVaR Optimization}

\label{subs:Convergence}
Next, we show that if both players play the no-regret algorithms discussed above, they solve the game~\eqref{eq:DROLearning}. Using $J(\theta, q) \coloneqq q^\top L(\theta)$, 
the minimax equilibrium of the game is the point $(\theta^\star, q^\star)$ such that  $\forall \theta\in \Theta,q\in \mathcal{Q}^\alpha;~J(\theta^\star, q) \leq J(\theta^\star, q^\star) \leq J(\theta, q^\star) $.
We assume that this point exists (which is guaranteed, e.g., when the sets $\mathcal{Q}^\alpha$ and $\Theta$ are compact).
The game regret is
\begin{equation}
    \operatorname{GameRegret}_T \coloneqq \sum \nolimits_{t=1}^T J(\theta_t, q^\star) - J(\theta^\star, q_t).\label{eq:GameRegret}
\end{equation}
\begin{theorem}[Game Regret] \label{thm:no-regret}
    Let $L_i(\cdot): \Theta \to [0, 1]$, $i=\left\{1, ..., N \right\}$ be a fixed set of loss functions.
    If the sampler plays \adacvar and the learner plays the oracle-BTL algorithm, then the game has regret $O(\sqrt{TN\log N} + \epsilon_{\mathrm{oracle}}T)$.
\end{theorem}
\begin{proof}[Proof sketch] We bound the Game regret with the sum of the learner and sampler regret and use the results of \Cref{lemma:Learner} and \Cref{lemma:Sampler}. For a detailed proof please refer to \Cref{proof:no-regret}.
\end{proof}
%
This immediately implies our main theoretical result, namely a performance guarantee for the solution obtained by \adacvar for the central problem of minimizing the empirical CVaR~\eqref{eq:DROLearning}.
\begin{corollary}[Online to Batch Conversion] \label{cor:excess-cvar}
Let $L_i(\cdot): \Theta \to [0, 1]$, $i=\left\{1, ..., N \right\}$ be a set of loss functions sampled from a distribution $\mathcal{D}$.
Let $\theta^\star$ be the minimizer of the CVaR of the empirical distribution $\hat{\CVar}^\alpha$.
Let $\bar{\theta}$ be the output of \adacvar, selected uniformly at random from the sequence $\left\{\theta_t\right\}_{t=1}^T$.
Its expected excess CVaR is bounded as:
\begin{equation*}
    \mathbb{E} \hat{\CVar}^\alpha[L(\bar{\theta})] \leq  \hat{\CVar}^\alpha[L(\theta^\star)] + O(\sqrt{N\log N / T}) + \epsilon_{\mathrm{oracle}}
\end{equation*}
where the expectation is taken w.r.t.~the randomization in the algorithm, both for the sampling steps and the randomization in choosing $\bar{\theta}$.
\end{corollary}
\begin{proof}[Proof sketch]
For a detailed proof please refer to \Cref{proof:excess-cvar}.
The excess CVaR is bounded by the duality gap, which in turn is upper-bounded by the average game regret.
\end{proof}

\begin{corollary}[Population Guarantee] \label{cor:excess-population-cvar}
Let $L(\cdot): \Theta \to [0, 1]$ be a Random Variable induced by the data distribution $\mathcal{D}$.
Let $\theta^\star$ be the minimizer of the CVaR at level $\alpha$ of such Random Variable.
Let $\bar{\theta}$ be the output of \adacvar, selected uniformly at random from the sequence $\left\{\theta_t\right\}_{t=1}^T$.
Then, with probability at least $\delta$ the expected excess CVaR of $\bar{\theta}$ is bounded as:
\begin{equation*}
    \mathbb{E} \CVar^\alpha[L(\bar{\theta})] \leq  \CVar^\alpha[L(\theta^\star)] + O(\sqrt{N\log N / T}) + \epsilon_{\mathrm{oracle}} + \epsilon_{\mathrm{stat}}
\end{equation*}
where $\epsilon_{\mathrm{stat}} = \tilde{O}(\frac{1}{\alpha}\sqrt{\frac{1}{N}})$ comes from the statistical error and the expectation is taken w.r.t.~the randomization in the algorithm, both for the sampling steps and the randomization in choosing $\bar{\theta}$.
\end{corollary}
\begin{proof}[Proof sketch]
For a detailed proof please refer to \Cref{proof:excess-population-cvar}.
We bound the statistical error using \Cref{prop:uniform-convergence} and the optimization error is bounded using \Cref{cor:excess-cvar}.
\end{proof}

\looseness -1 It is instructive to consider the special cases $k=1$ and $k=N$. For $k=N$, $q_t$ remains uniform and \adacvar reduces to SGD. For $k=1$, the sampler simply plays standard EXP.3 over data points and \adacvar reduces to the algorithm of \citet{shalev2016minimizing} for the max loss.

\section{Experiments} \label{sec:Experiments}
In our experimental evaluation, we compare \adacvar on both convex (linear regression and classification) and non-convex (deep learning) tasks.  In addition to studying how it performs in terms of the CVaR and empirical risk on the training and test set, we also investigate to what extent it can help guard against distribution shifts. In \Cref{app:ExperimentalSetup}, we detail the experimental setup.
We provide an open-source implementation of our method, which is available at \url{http://github.com/sebascuri/adacvar}.
%
\paragraph{Baseline Algorithms} \looseness -1
We compare our adaptive sampling algorithm (\adacvar) to three baselines:
first, an i.i.d.~sampling scheme that optimizes Problem~\eqref{eq:BadLearning} using the truncated loss (\trunkcvar);
second, an i.i.d.~sampling scheme that uses a smoothing technique to relax the $\sum_{i}[x_i]_{+}$ non-linearity (\softcvar). 
\citet{tarnopolskaya2010cvar} compare different smoothing techniques for the $\sum_{i}[x_i]_{+}$ non-linearity. 
Of these, we use the relaxation $T \log(\sum_{i}e^{x_i/T})$ proposed by \citet{nemirovski2006convex}. In each iteration, we heuristically approximate the population sum with a mini-batch.
Third, we also compare a standard i.i.d.~sampling ERM scheme that stochastically minimizes the average of the losses (\mean).

\subsection{Convex CVaR Optimization} \label{exp:convex}
\looseness -1 We first compare the different algorithms in a controlled convex setting, where the classical \trunkcvar algorithm is expected to perform well. 
We consider three UCI regression data sets, three synthetic regression data sets, and eight different UCI classification data sets \citep{Dua:2019}.
The left and middle plots in \Cref{fig:convex_summary} present a summary of the results (see \Cref{tab:ConvexCvar} in \Cref{app:DetailedExperiments} for a detailed version). We evaluate the CVaR ($\alpha=0.01$) and average loss for linear regression and the accuracy, and CVaR and average surrogate loss for classification (logistic regression) on the test set.
In linear regression, \adacvar performs comparably or better to benchmarks in terms of the CVaR of the loss and is second best in terms of the average loss.
In classification, \trunkcvar performs better in terms of the CVaR for the {\em surrogate loss} but performs {\em poorly} in terms of accuracy. 
This is due to the fact that it learns a predictive distribution that is close to uniform.
\adacvar has a comparable accuracy to ERM (\mean algorithm) but a much better CVaR. 
Hence, it finds a good predictive model while successfully controlling the prediction risk.
\begin{figure}[t]
    \includegraphics[width=\columnwidth]{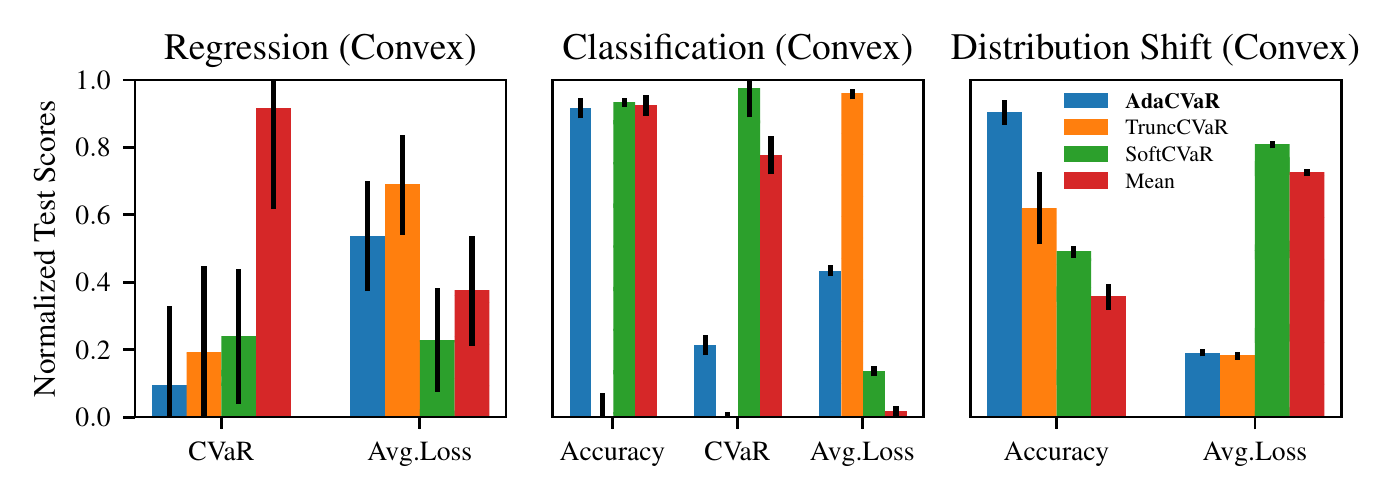}
    \caption{\looseness -1 
    Scores are normalized between 0 and 1 to compare different data sets.
    \textit{Left}: Linear regression tasks. \adacvar has lower CVaR than benchmarks.
    \textit{Middle}: Binary classification (logistic regression) tasks. \adacvar obtains the same accuracy as \mean and \softcvar with lower CVaR. \trunkcvar outputs an approximately uniform distribution yielding low CVaR but poor predictive accuracy.
    \textit{Right}: Binary classification (logistic regression) tasks with train/test 90\% distribution shift.
    \adacvar has the highest test accuracy and low average surrogate loss.}
    \label{fig:convex_summary}
\end{figure}
\subsection{Convex CVaR Distributional Robustness} \label{exp:convex-dist}
\looseness -1 We use the same classification data sets and classifiers as in Section \ref{exp:convex}.
To produce the distribution shift, we randomly sub-sample the majority class in the training set, so that the new training set has a 10\%/90\% class imbalance and the majority/minority classes are inverted. 
The test set is kept unchanged. 
Such shifts in class frequencies are common \citep[Section 3.2]{mehrabi2019survey}.

\looseness -1 We consider $\alpha=0.1$, which is compatible with the data imbalance.
The right plot in \Cref{fig:convex_summary}
shows a summary of the results (See \Cref{tab:ShiftTest} in \Cref{app:DetailedExperiments} for detailed results).
\adacvar has higher test accuracy than the benchmarks and is comparable to \trunkcvar on average log-likelihood.

\looseness -1 We note that the CVaR provides robustness with respect to worst-case distribution shifts. 
Such a \emph{worst case} distribution might be too pessimistic to be encountered in practice, however.
Instead, \adacvar appears to benefit from the varying distributions during training
and protects better against non-adversarial distribution shifts. 
Other techniques for dealing with imbalanced data might also be useful to address this distribution shift empirically, but are only useful if there is an a-priori knowledge of the class ratio in the test set.
Instead, the CVaR optimization guards against any distribution shift. 
Furthermore, with such a-priori knowledge, such techniques can also be used together with \adacvar. 
We provide extended experiments analyzing distribution shift in \Cref{app:exp:dist-shift}.


\subsection{Non-Convex (Deep Learning) CVaR Optimization} \label{exp:deep-cvar}
\begin{figure}[t]
    \includegraphics[width=\columnwidth]{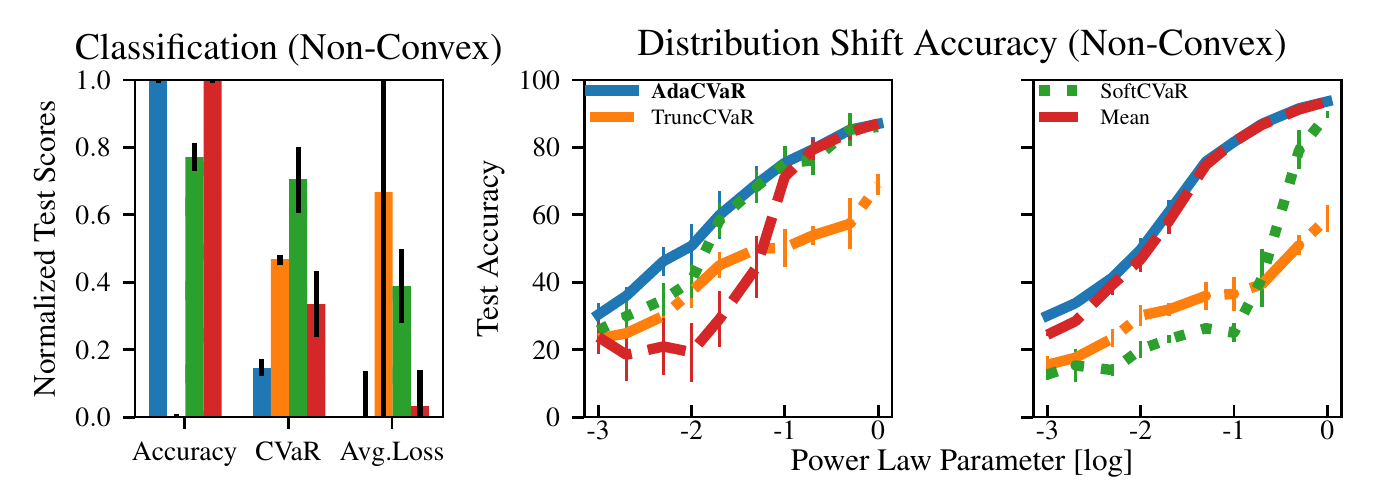}
    \caption{Non Convex Optimization tasks. 
    \textit{Left}: Normalized scores in image classification tasks. \adacvar attains state-of-the-art accuracy and lowest CVaR.
    \textit{Middle and Right}: Test accuracy under train/test distribution shift on CIFAR-10 for VGG16-BN (middle) and ResNet-18 (right). 
    Lower $\beta$ indicates larger shift.
    \adacvar has always better test accuracy than benchmarks.
    }
    \label{fig:non_convex_summary}
\end{figure}

\looseness -1 We test our algorithm on common non-convex optimization benchmarks in deep learning (MNIST, Fashion-MNIST, CIFAR-10). 
As it is common in these setting, we perform data-augmentation on the training set. 
Thus, the effective training set size is infinite. 
To address this, we consider a mixture of distributions in a similar spirit as \citet{borsos2019online}. 
Each data point serves as a representative of a distribution over all its possible augmentations. 
We optimize the CVaR of this mixture of distributions as a surrogate of the CVaR of the infinite data set.
The left plot in \Cref{fig:non_convex_summary} summarizes the results (See \Cref{tab:VisionTest} in \Cref{app:DetailedExperiments}).
\adacvar reaches the same accuracy as ERM in all cases and has lower CVaR. 
Only in CIFAR-10 it does not outperform \trunkcvar in terms of the CVaR of the surrogate loss. 
This is because the \trunkcvar yields a predictive model that is close to uniform.
Instead, \adacvar still yields useful predictions while controlling the CVaR.

\paragraph{Gradient Magnitude and Training Time} \looseness -1 The gradients of \trunkcvar are either 0 or $1/\alpha$ times larger than the gradients of the same point using \mean.
A similar but smoothed phenomenon arises with \softcvar. This makes training these losses considerably harder due to exploding gradients and noisier gradient estimates. 
With the same learning rates, these algorithms usually produce numerical overflows and, to stabilize learning, we used considerably smaller learning rates. 
In turn, this increased the number of iterations required for convergence. 
\adacvar does not suffer from this as the gradients have the same magnitude as in \mean.
For example, to reach 85 \% \emph{train} accuracy \adacvar requires 7 epochs, \mean 9, \softcvar 21, and \trunkcvar never surpassed 70 \% train accuracy. 
There was no significant difference between time per epoch of each of the algorithms.

\subsection{Distributional Robustness in Deep Learning through Optimizing the CVaR} \label{exp:deep-robust}
\looseness -1 Lastly, we demonstrate that optimizing the CVaR yields improved robustness to distribution shifts in deep learning.
We simulate distribution shift through mismatching training and test class frequencies. Since we consider multi-class problems, we simulate power-law class frequencies, which are commonly encountered in various applications \citep{clauset2009power}. More specifically, we sub-sample each class of the training set of CIFAR-10 so that the class size follows a power-law distribution $p(|c|) \propto c^{\log{\beta}}$, where $|c|$ is the size of the $c$-th class and keep the test set unchanged. 
In middle and right plots of \Cref{fig:non_convex_summary}, we show the test accuracy for different values of $\beta$ for VGG16-BN and ResNet-18 networks.
The algorithms do not know a-priori the amount of distribution shift to protect against and consider a fixed $\alpha=0.1$.
For all distribution shifts, \adacvar is superior to the benchmarks.  

When a high-capacity network learns a perfectly accurate model, then the average and CVaR of the loss distribution have both zero value. 
This might explain the similarity between \mean and \adacvar for ResNet-18. 
Instead, there is a stark discrepancy between \adacvar and \mean in VGG16. 
This shows the advantage of training in a risk averse manner, particularly when the model makes incorrect predictions due to a strong inductive bias.

\section{Conclusions}
\looseness -1 The CVaR is a natural criterion for training ML models in a risk-aware fashion. 
As we saw, the traditional way of optimizing it via truncated losses fails for modern machine learning tasks due to high variance of the gradient estimates.
Our novel adaptive sampling algorithm \adacvar exploits the distributionally robust optimization formulation of the CVaR, and tackles it via regret minimization.
It naturally enables the use of standard stochastic optimization approaches (e.g., SGD), applied to the marginal distribution of a certain k-DPP.
Finally, we demonstrate in a range of experiments that \adacvar is superior to the \trunkcvar algorithm for regression and classification tasks, both in convex and non-convex learning settings.
Furthermore, \adacvar provides higher robustness to (non-adversarial) distribution shifts than \trunkcvar, \softcvar, or \mean algorithms.

\section*{Broader Impact}
Increasing {\em reliability} is one of the central challenges when deploying machine learning in high-stakes applications. We believe our paper makes important contributions to this endeavor by going beyond simply optimizing the {\em average} performance, and considering risk in deep learning.
The CVaR is also known to be an avenue towards enforcing {\em fairness constraints} in data sets \citep{williamson2019fairness}.
Hence, our algorithm also contributes to optimizing fair deep models, by counteracting inherent biases in the data (e.g., undersampling of certain parts of the population). 


\begin{ack}
This project has received funding from the European Research Council (ERC) under the European Unions Horizon 2020 research, innovation programme grant agreement No 815943, a DARPA YFA award, and NSF CAREER award 1553284. 
\end{ack}

\newpage
\bibliography{bibliography}

\newpage
\appendix
\section{Approximate Sampling from k-DPP Marginals} \label{subs:Sampling}

The final piece of the \adacvar algorithm is how to efficiently compute the marginal distribution $\P_w(i)$ of the k-DPP model, and how to sample from it. 

\paragraph{Cost of sampling}
Compared to general k-DPPs, our setting has the crucial advantage that the kernel matrix is diagonal. Thus there is no need for performing an expensive eigendecomposition of the kernel matrix which is required in the general case.
The marginals of diagonal k-DPPs are 
\begin{align}
\P_w(i) = w_i\frac{e^{k-1}_{-i}}{e^k_N},  \label{eq:exact_marginals}
\end{align}
where ${e^k_N} = \sum_{|I|=k} \prod_{i \in I} w_i$ is the elementary symmetric polynomial of size $k$ for the ground set $[N]$ and $e^{k-1}_{-i}$ is the elementary symmetric polynomial of size $k-1$ for the ground set $[N]\setminus i$.
Unfortunately, naively computing the elementary symmetric polynomials has a complexity of $O(N^2k) = O(\alpha N^3)$ using \citep[Algorithm 7]{kulesza2012determinantal}.
Even if this computation could be performed fast, exact computation of the elementary symmetric polynomials is numerically unstable.

\looseness -1 In view of this, \citet{barthelme2019asymptotic} propose an approximation to k-DPPs valid for large-scale ground sets which has better numerical properties.
Their main idea is to relax the sample size constraint of the k-DPP with a soft constraint such that the \emph{expected} sample size of the matched DPP is $k$.
The total variation distance between the marginal probabilities of the k-DPP and DPP decays as $O(1/N)$.
The marginal probabilities of this matched DPP are simply
\begin{equation}
    \hat{\P}_w(i) = \frac{w_i e^\nu}{1+w_i e^\nu}, \label{eq:approx_marginals}
\end{equation}
where $\nu$ softly enforces the sample size constraint $\sum_{i=1}^N \frac{w_i e^\nu}{1+w_i e^\nu} = k$.
For a given $\nu$, we can efficiently sample from this singleton-marginal distribution, which takes $O(\log(N))$ using the same sum-tree data-structure as \citet{shalev2016minimizing}. 

\begin{proposition}[Excess Regret of Approximate Sampler]
Let $\tilde{q}$ be the marginals of the approximate k-DPP \eqref{eq:approx_marginals} and $q$ the exact marginals of the k-DPP \eqref{eq:exact_marginals}. Then, a sampler that plays \adacvar using the approximate k-DPPs marginals suffers an extra regret of $\epsilon_{\mathrm{approx}}T$, with $\epsilon_{\mathrm{approx}} = O(1/N)$, for large enough $N$.
\end{proposition}
\begin{proof}
Let $\widetilde{\operatorname{SR}}_T$ be the regret of the player that selects the approximate k-DPPs marginals. Then,
\begin{align*}
    \widetilde{\operatorname{SR}}_T &\coloneqq  \max_{q \in \mathcal{Q}^\alpha} \sum_{t=1}^T q^\top L_t - \sum_{t=1}^T \tilde{q}_t^\top L_t \\
    &= \operatorname{SR}_T + \sum_{t=1}^T (\tilde{q}_t - q_t)^\top L_t \\
    &\leq \operatorname{SR}_T + \sum_{t=1}^T \| \tilde{q}_t - q_t \|_1 \| L_t\|_\infty \\
    &\leq  \operatorname{SR}_T + \sum_{t=1}^T \| \tilde{q}_t - q_t \|_1 \\ 
    &\leq  \operatorname{SR}_T + \epsilon_{\mathrm{approx}} T
\end{align*}
where the first inequality holds due to H\"older's inequality, the second inequality due to $L \in [0, 1]$, and finally noticing that $\| \tilde{q}_t - q_t \|_1$  is proportional to the total variation distance. Using Theorem 2.1 from \citet{barthelme2019asymptotic} the results follow.
\end{proof}

We note that this result implies an an extra $O(1/N)$ term in the excess CVaR bounds from \Cref{cor:excess-cvar,cor:excess-population-cvar}. This is fast compared to statistical $O(1/\sqrt{N})$ rates. 
For small $N$, the exact marginals can be computed efficiently so there is no need for the approximation.
\section{Omitted Proofs} 
\subsection{Proof of Proposition \ref{prop:uniform-convergence}} \label{proof:uniform-convergence}

\newtheorem*{prop_repeat}{Proposition \ref{prop:uniform-convergence}}
\begin{prop_repeat} 
    Let $h: \mathcal{X} \to \mathcal{Y}$ be a function class with finite VC-dimension $|\mathcal{H}|$.
    Let $L(h): \mathcal{H} \to [0, 1]$ be a random variable.
    Then, for any $0 < \alpha \leq 1$, with probability at least $1-\delta$, 
    \begin{align*} 
        \E\left[ \sup_{h \in \mathcal{H}} \left| \widehat{\CVar}^\alpha [L(h)] - \CVar^\alpha[L(h)] \right| \right] \leq \frac{1}{\alpha}\sqrt{\frac{\log(2 |\mathcal{H}| /\delta)}{N}}.
    \end{align*}
\end{prop_repeat}

\begin{proof}[Proof for \Cref{prop:uniform-convergence}]
    \citet{brown2007large} proves that the following two inequalities hold jointly with probability $1-\delta$, $\delta \in (0, 1]$ for a single $h \in \mathcal{H}$: 
    \begin{align*}
        \CVar^\alpha(L(h)) &\geq \widehat{\CVar}^\alpha (L(h)) - \frac{1}{\alpha}\sqrt{\frac{\log(2/\delta)}{N}} \\ 
        \CVar^\alpha(L(h)) &\leq \widehat{\CVar}^\alpha (L(h)) + \sqrt{\frac{5\log(6/\delta)}{\alpha N}}
    \end{align*}
    Taking the union bound over all $h \in \mathcal{H}$: 
    \begin{align*}
        \CVar^\alpha(L(h)) &\geq \widehat{\CVar}^\alpha (L(h)) - \frac{1}{\alpha}\sqrt{\frac{\log(2|\mathcal{H}|/\delta)}{N}} \\ 
        \CVar^\alpha(L(h)) &\leq \widehat{\CVar}^\alpha (L(h)) + \sqrt{\frac{5\log(6|\mathcal{H}|/\delta)}{\alpha N}}
    \end{align*}
    The theorem follows from taking the maximum between lower and upper bounds. 
    
    If $\mathcal{H}$ is a non-finite class we follow a standard argument with bounded norm of real valued functions, i.e., $\|\mathcal{H}\|_{\infty} \leq B$, we rely on the standard argument based on covering numbers.
    To define such covering, we fix $\epsilon > 0$ and consider a set $\mathcal{C}_{\mathcal{H}, \epsilon}$ of minimum cardinality, such that for all $h \in \mathcal{H}$, there exists an $h' \in \mathcal{C}_{\mathcal{H}, \epsilon}$, satisfying $|h - h'| \leq \epsilon$. Define the covering number as $\mathcal{N}_{\mathcal{H},\epsilon} = |\mathcal{C}_{\mathcal{H}, \epsilon}|$, then taking the union bound for $\epsilon=1\sqrt{N}$ we arrive to the result 
    \begin{align*} 
        \E\left[ \sup_{h \in \mathcal{H}} \left| \widehat{\CVar}^\alpha [L(h)] - \CVar^\alpha[L(h)] \right| \right] \leq \frac{1}{\alpha}\sqrt{\frac{\log(2 \mathcal{N}_{\mathcal{H},\epsilon} /\delta)}{N}}.
    \end{align*}
    
    Finally, the logarithm of the covering number can be upper bounded using the pseudo-VC dimension of the function class, for a proof see \citet[Chapter 29]{devroye2013probabilistic}.
    
\end{proof}

\subsection{Proof of Lemma \ref{lemma:Sampler}} \label{proof:Sampler}
\newtheorem*{lemma_sampler_repeat}{Lemma \ref{lemma:Sampler}}

\begin{lemma_sampler_repeat}
Let the sampler player play the \kexpthree Algorithm with $\eta = \sqrt{\frac{\log N}{N T}}$.
Then, she suffers a sampler regret \eqref{eq:SamplerRegret} of at most $O(\sqrt{TN\log N})$.
\end{lemma_sampler_repeat}

In order to prove this, we need to first show that \kexpthree is a valid algorithm for the sampler player. This we do next. 
\begin{proposition} \label{prop:KDPP_DRO}
    The marginals of any k-DPP with a diagonal matrix kernel $K = \diag(w)$ are in the set 
    $\mathcal{Q}^\alpha_k = \left \{ k q \in \real[N] \mid q \in \mathcal{Q}^\alpha \right \} $.  
\end{proposition}
\begin{proof}
    For any $w \in \real[N]_{\geq 0}$ the marginals of the k-DPP with kernel $K = \diag(w)$ are:
    \begin{equation}
        \P_w(i) = \sum_{I \ni i} \P_w(I) = \frac{ \sum_{I \ni i} \prod_{i' \in I} w_{i'} }{\sum_{I} \prod_{i' \in I} w_{i'}}. \label{eq:marginals}
    \end{equation}
    From \cref{eq:marginals}, clearly $0 \leq \P_w(i) \leq 1$. Summing \cref{eq:marginals} over $i$ we get:
    \begin{equation*}
        \sum_i \P_w(i) = \sum_i \sum_{I \ni i} \P_w(I) = \sum_{I} \P_w(I) \sum_{i \in I}1  = k
    \end{equation*}
    This shows that $\P_w(i) \in \mathcal{Q}^\alpha_k$.
\end{proof}

\begin{proposition} \label{prop:comparator}
    Let $\tilde{L}_I = \sum_{i \in I} L_i$. 
    Let $\Delta = \left\{ \tilde{w} \in \real[\binom{N}{k}] \mid, 0 \leq w_I \leq 1, \sum_{I} w_I = 1 \right\}$ the set of distributions over the $\binom{N}{k}$ subsets of size $k$ of the ground set $[N]$. 
    
    \begin{equation}
        \max_{q \in \mathcal{Q}^\alpha} \sum_{t=1}^T q^\top L_t = \max_{\tilde{w} \in \Delta } \frac{1}{k} \sum_{t=1}^T \tilde{w} \tilde{L}_I \label{eq:comparator}
    \end{equation}
\end{proposition}
\begin{proof} 
    Both left and right sides of \eqref{eq:comparator} are linear programs over a convex polytope, hence the solution is in one of its vertices \citep{murty1983linear}.
    The vertices of $\mathcal{Q}^\alpha$ are vectors $\frac{1}{k}\mathbf{1}_I = \frac{1}{k}[[i \in I]]$.
    These vectors have $\frac{1}{k}$ in coordinate $i$ if the coordinate belongs to set $I$ and 0 otherwise.
    The vertices of the simplex are just $[[I]]$, one for coordinate $I$.
    
    Let $q^\star = \frac{1}{k}\mathbf{1}_{I^\star}$ be the solution of the l.h.s.~of \eqref{eq:comparator}.
    Assume that $\hat{I} \neq I^\star$ is the solution of the right hand side. This implies that $\tilde{L}_{\hat{I}} \geq \tilde{L}_{I^\star}$. Therefore, $\sum_{i \in \hat{I}} L_i \geq \sum_{i \in I^\star} L_i$. This in turn implies that $\mathbf{1}_{\hat{I}} L_i \geq \mathbf{1}_{I^\star} L_i$, which contradicts the first predicate. In the case the equalities hold, then the values l.h.s~and r.h.s.~ of equation \eqref{eq:comparator} are also equal. 
\end{proof}

\begin{proof}[Proof of \Cref{lemma:Sampler}]
    \begin{align*}
        \operatorname{SR}_T &= \max_{q \in \mathcal{Q}^\alpha} \sum_{t=1}^T q^\top L_t - \sum_{t=1}^T q_t^\top L_t \\ 
        &= \max_{\tilde{w} \in \Delta } \frac{1}{k} \sum_{t=1}^T \tilde{w} \tilde{L}_I - \sum_{t=1}^T q_t^\top L_t \\
        &= \frac{1}{k} \left( \max_{\tilde{w} \in \Delta} \sum_{t=1}^T \tilde{w} \tilde{L}_I - \sum_{t=1}^T \sum_{i} \P_{w_t}(i) L_i \right)\\
        &= \frac{1}{k} \left( \max_{\tilde{w} \in \Delta} \sum_{t=1}^T \tilde{w} \tilde{L}_I - \sum_{t=1}^T \sum_{I} \P_{w_t}(I) \tilde{L}_I \right)\\
        &\leq O(\sqrt{NT\log(N)}) \tag*{\qedhere}
    \end{align*}
    The first equality uses \Cref{prop:comparator}. The second equality uses \Cref{prop:KDPP_DRO} and the fact that the iterates $q_t$ come from the \kexpthree algorithm. 
    The third equality uses the definition of $\tilde{L}$. 
    The final inequality is due to \citet[Lemma 1]{alatur2020multi}.
\end{proof}

    

\subsection{Proof of Lemma \ref{lemma:Learner}} \label{proof:Learner}
\newtheorem*{lemma_learner_repeat}{Lemma \ref{lemma:Learner}}
\begin{lemma_learner_repeat}
A learner player that plays the BTL algorithm with access to an ERM oracle as in \Cref{assumption:ERM}, achieves at most $\epsilon_{\mathrm{oracle}}T$ regret.
\end{lemma_learner_repeat}

\begin{proof}
    We proceed by induction.
    Clearly for $T=1$, $\operatorname{LR}_{1} \leq \epsilon_{\mathrm{oracle}}$.
    Assume true for $T-1$, the inductive hypothesis is $\operatorname{LR}_{T-1} \leq\epsilon_{\mathrm{oracle}} (T-1)$.
    The regret at time $T$ is:
    \begin{align*}
    	&\operatorname{LR}_{T} = \sum_{t=1}^{T} q_t^\top L(\theta_t) - \min_{\theta \in \Theta} \sum_{t=1}^{T} q_t^\top L(\theta), \\
    	&\leq \sum_{t=1}^{T} q_t^\top \left[L(\theta_t) -  L(\theta_T) \right] + \epsilon_{\mathrm{oracle}}\\ &= \sum_{t=1}^{T-1} q_t^\top \left[L(\theta_t) -  L(\theta_T) \right] + \epsilon_{\mathrm{oracle}}\\
    	&\leq \operatorname{LR}_{T-1} + \epsilon_{\mathrm{oracle}} \leq \epsilon_{\mathrm{oracle}}T \tag*{\qedhere}
    \end{align*}
    
    The first inequality is due to the definition of the oracle, the second inequality is by definition of the minimum in the learner regret, and the final inequality is due to the inductive hypothesis.
\end{proof}

\subsection{Proof of Theorem \ref{thm:no-regret}} \label{proof:no-regret}

\newtheorem*{thm_no-regret_repeat}{Theorem \ref{thm:no-regret}}

\begin{thm_no-regret_repeat}
\label{thm:no-regret-rep}
    Let $L_i(\cdot): \Theta \to [0, 1]$, $i=\left\{1, ..., N \right\}$ be a fixed set of loss functions.
    If the sampler player uses \adacvar and the learner player uses the BTL algorithm, then the game has regret $O(\sqrt{TN\log N})$.
\end{thm_no-regret_repeat}
\begin{proof} We decompose bound the game regret into the sum of player and sampler regret. 
    To bound the regret, we bound it with the sum of the Learner and Sampler regret as follows:
    \begin{align*}
        \operatorname{GameRegret}_T &= \sum\nolimits_{t=1}^T J(\theta_t, q^\star) - J(\theta^\star, q_t), \\
        &\leq \max_{q \in \mathcal{Q}^\alpha} \sum\nolimits_{t=1}^T J(\theta_t, q) - J(\theta^\star, q_t), \\ 
        (\text{\Cref{lemma:Learner}})\; &\leq \max_{q \in \mathcal{Q}^\alpha} \sum\nolimits_{t=1}^T J(\theta_t, q) - J(\theta_t, q_t), \\
        (\text{\Cref{lemma:Sampler}})\;  &\leq O(\sqrt{TN\log N} + \epsilon_{\mathrm{oracle}}T).  \tag*{\qedhere}
    \end{align*}
\end{proof}

\subsection{Proof of Corollary \ref{cor:excess-cvar}} \label{proof:excess-cvar}
\newtheorem*{cor_excess-cvar_repeat}{Corollary \ref{cor:excess-cvar}}
\begin{cor_excess-cvar_repeat}[Online to Batch Conversion] \label{cor:excess-cvar-rep}
Let $L_i(\cdot): \Theta \to [0, 1]$, $i=\left\{1, ..., N \right\}$ be a set of loss functions sampled from a distribution $\mathcal{D}$.
Let $\theta^\star$ be the minimizer of the CVaR of the empirical distribution $\hat{\CVar}^\alpha$.
Let $\bar{\theta}$ be the output of \adacvar, selected uniformly at random from the sequence $\left\{\theta_t\right\}_{t=1}^T$.
Its expected excess CVaR is bounded as:
\begin{equation*}
    \mathbb{E} \hat{\CVar}^\alpha[L(\bar{\theta})] \leq  \hat{\CVar}^\alpha[L(\theta^\star)] + + O(\sqrt{N\log N / T}) + \epsilon_{\mathrm{oracle}}
\end{equation*}
where the expectation is taken w.r.t.~the randomization in the algorithm, both for the sampling steps and the final randomization in choosing $\bar{\theta}$.
\end{cor_excess-cvar_repeat}
\begin{proof}
We use the proof techinque from \citet{agarwal2018learning}[Theorem 3]. 
The CVaR of $\bar{\theta}$ is $\hat{\CVar}^\alpha[L(\bar{\theta})] = \max_{q \in \mathcal{Q^\alpha}} J(\bar{\theta}, q)$. Let $\bar{q}^\star \coloneqq \arg \max_{q \in \mathcal{Q^\alpha}} J(\bar{\theta}, q)$. For any $q \in \mathcal{Q^\alpha}$, we can use the sampler regret to bound:
\begin{align*}
    \E\left[J(\bar{\theta}, q)\right] &= \E \left[\frac{1}{T} \sum_{t=1}^T J(\theta_t, q) \right] \\
    &\leq \E \left[\frac{1}{T} \sum_{t=1}^T J(\theta_t, q_t) \right] + \frac{1}{T}\operatorname{SR}_{T}
\end{align*}
Likewise, for any $\theta \in \Theta$,
\begin{align*}
    \E\left[J(\theta, \bar{q})\right] &= \E \left[\frac{1}{T} \sum_{t=1}^T J(\theta, q_t) \right] \\
    &\geq \E \left[\frac{1}{T} \sum_{t=1}^T J(\theta_t, q_t) \right] - \frac{1}{T}\operatorname{LR}_{T}
\end{align*}
Using the minimax inequality, we bound the excess risk for any $q \in \mathcal{Q}$ and $\theta \in \Theta$ by the average game regret:
\begin{align*}
    \E\left[J(\bar{\theta}, q)\right] & \leq  J(\theta^\star, q^\star) +  \E\left[J(\bar{\theta}, q) - J(\theta, \bar{q}) \right] \\ 
    &\leq J(\theta^\star, q^\star) + \underbrace{\frac{1}{T} \left(\operatorname{LR}_{T} + \operatorname{SR}_{T} \right)}_{\epsilon}
\end{align*}
Noting that, $\operatorname{GameRegret}_T = \operatorname{LR}_{T} + \operatorname{SR}_{T}$ and that the latter results holds in particular for $\bar{q}^\star$. Furthermore, the pair $(\bar{\theta}, \bar{q})$ is an $\epsilon$-equilibrium point of the game in the sense:
$J(\theta^\star, q^\star) - \epsilon \leq \E\left[J(\bar{\theta}, \bar{q})\right] \leq J(\theta^\star, q^\star) + \epsilon$.
\end{proof}

\subsection{Proof of Corollary \ref{cor:excess-population-cvar}} \label{proof:excess-population-cvar}
\newtheorem*{cor_excess-pop-cvar_repeat}{Corollary \ref{cor:excess-population-cvar}}
\begin{cor_excess-pop-cvar_repeat}[Online to Batch Conversion] \label{cor:excess-population-cvar-rep}
Let $L(\cdot): \Theta \to [0, 1]$ be a Random Variable induced by the data distribution $\mathcal{D}$.
Let $\theta^\star$ be the minimizer of the CVaR at level $\alpha$ of such Random Variable.
Let $\bar{\theta}$ be the output of \adacvar, selected uniformly at random from the sequence $\left\{\theta_t\right\}_{t=1}^T$.
Then, with probability at least $\delta$ the expected excess CVaR of $\bar{\theta}$ is bounded as:
\begin{equation*}
    \mathbb{E} \CVar^\alpha[L(\bar{\theta})] \leq  \CVar^\alpha[L(\theta^\star)] + O(\sqrt{N\log N / T}) + \epsilon_{\mathrm{oracle}} + \epsilon_{\mathrm{stat}}
\end{equation*}
where $\epsilon_{\mathrm{stat}} = \tilde{O}(\frac{1}{\alpha}\sqrt{\frac{1}{N}})$ comes from the statistical error and the expectation is taken w.r.t.~the randomization in the algorithm, both for the sampling steps and the randomization in choosing $\bar{\theta}$.
\end{cor_excess-pop-cvar_repeat}
\begin{proof}
\begin{align*}
    \mathbb{E} \CVar^\alpha[L(\bar{\theta})] &\leq \mathbb{E} \hat{\CVar}^\alpha[L(\bar{\theta})] + \epsilon_{\mathrm{stat}} \\
    & \leq \mathbb{E} \hat{\CVar}^\alpha[L(\theta^\star)] + O(\sqrt{N\log N / T}) + \epsilon_{\mathrm{oracle}} + \epsilon_{\mathrm{stat}}  \\
    & \leq \mathbb{E} \CVar^\alpha[L(\theta^\star)] + O(\sqrt{N\log N / T}) + \epsilon_{\mathrm{oracle}}  + 2\epsilon_{\mathrm{stat}},
\end{align*}
Where the first and last inequality hold by the uniform convergence results from \Cref{prop:uniform-convergence} and the second inequality by \Cref{cor:excess-cvar}.
\end{proof}

\section{Learner Player Algorithm for Convex Losses}  \label{app:ConvexLearner}

In the convex setting, there are online learning algorithms that have no-regret guarantees and there is no need to play the BTL algorithm \eqref{eq:BTL} exactly. Instead, stochastic gradient descent (SGD) \citet{zinkevich2003online} or online mirror descent (OMD) \citep{beck2003mirror} both have no-regret guarantees. 
We focus now on SGD but, for certain geometries of $\Theta$ and appropriate mirror maps, OMD has exponentially better regret guarantees (in terms of the dimension of the problem).

\begin{algorithm}[ht]
 \caption{Ada-CVaR-CVX} \label{alg:CVXGame}
\begin{algorithmic}[1]
    \INPUT Learning rates $\eta_s$, $\eta_l$. 
    \STATE \textbf{Sampler:} Initialize k-DPP $w_1 = \mathbf{1}_N$.
    \STATE \textbf{Learner:} Initialize parameters $\theta_1 \in \Theta$.
    \FOR{$t=1,\ldots, T$}
        \STATE \textbf{Sampler:} Sample element $i_t \sim q_{t} = \frac{1}{k}\P_{w_t}(i)$.
        \STATE \textbf{Sampler}: Build estimate $\hat{L}_{t} = \frac{L_{i_t}(\theta_{t})}{q_{t, i_t}} [[i == i_t]]$.
        \STATE \textbf{Sampler}: Update k-DPP $w_{t+1, i} = w_{t, i} e^{\eta_s \hat{L}_{t, i}}$.
        \STATE \textbf{Learner:} $\theta_{t+1} = \theta_{t} - \eta_l \nabla L_{i_t}(\theta_{t}) $.
    \ENDFOR
    \OUTPUT $\bar{\theta} = \frac{1}{T} \sum_{t=1}^T \theta_t$, $\bar{q} = \frac{1}{T} \sum_{t=1}^T q_t$
\end{algorithmic}
\end{algorithm}

\begin{lemma}
Let assume that $L_i(\cdot): \Theta \to [0, 1]$ be any sequence of convex losses, with $\| \nabla L_i \|_2 \leq G$ and $\|\Theta\|_2 \leq D$, then a learner player that plays SGD algorithm suffers at most regret $O(GD\sqrt{T})$.
\end{lemma}
\begin{proof}
    \citet[Chapter 3]{hazan2016introduction}. 
\end{proof}

Note that even if there are algorithms for the strongly convex case or exp-concave case that have $\log(T)$ regret, it does not bring any advantage in our case as the $\sqrt{T}$ term in the sampler regret dominates and is unavoidable \citep{audibert2013regret}. 

\begin{corollary} Let $L_i(\cdot): \Theta \to [0, 1]$ be any sequence of convex losses. Let the learner and sampler player play \Cref{alg:CVXGame}, then the game has regret $O(\sqrt{TN\log N} + \beta \sqrt{T})$, where $\beta$ is a problem-dependent constant.
\end{corollary}

This result changes the $\epsilon_{\mathrm{oracle}}$ term in \Cref{cor:excess-cvar,cor:excess-population-cvar} by a $O(1/\sqrt{T})$ term. 

\section{Experimental Setup} \label{app:ExperimentalSetup}

\paragraph{Implementation and Resources:}
We implemented all our experiments using PyTorch \citep{paszke2017automatic}. 
We ran our experiments convex experiments on an Intel(R) Xeon(R) CPU E5-2697 v4 \@ 2.30GHz machine. Our deep lerning experiments ran on an NVIDIA GeForce GTX 1080 Ti GPU. 

\paragraph{Datasets:} 
For classification we use the Adult, Australian Credit Approval, German Credit Data, Monks-Problems-1, Spambase, and Splice-junction Gene Sequences datasets from the UCI repository \citep{Dua:2019} and the Titanic Disaster dataset from \citep{eaton1995titanic}.
For regression we use the Boston Housing, and Abalone, and CPU small from the UCI repository, the sinc dataset is synthetic recreated from \citep{fan2017learning}, and normal and pareto datasets are synthetic datasets recreated from \citep{brownlees2015empirical} with Gaussian and Pareto noise, respectively.
For vision datasets, we use MNIST \citep{lecun1998gradient}, Fashion-MNIST \citep{xiao2017fashion}, and CIFAR-10 \citep{krizhevsky2014cifar}.

\paragraph{Models:} 
For convex regression and classification models we use linear models. 
For MNIST we use LeNet-5 neural network \citep{lecun1995convolutional}. 
For Fashion MNIST we use LeNet-5 using dropout \citep{hinton2012improving}. 
For CIFAR-10 we use ResNet18 \citep{he2016deep} neural network and VGG-16 \citep{simonyan2014very} with batch normalization \citep{ioffe2015batch}. 

\paragraph{Dataset Preparation:} 
We split UCI datasets into train, validation, and test set using a 50/30/20 split.
For vision tasks we use as validation set the same images in the train set, without applying data-augmentations.
For discrete categorical data, we use a one-hot-encoding.
We standarize continuous data.

\paragraph{Hyper-Parameter Search:} 
We ran a grid search over the hyperparameters for all the algorithms with a five different random seeds. In regression tasks, we selected the set of hyper-parameters with the lowest CVaR in the validation set. 
In classification tasks, we selected the set of hyper-parameters with the highest accuracy in the validation set. 
The hyperparameters are:
\begin{enumerate}
    \item Optimizer SGD with momentum or ADAM \citep{kingma2014adam}.
    \item Initial learning rates: \\ $\{0.05, 0.01, 0.005, 0.001, 0.0005, 0.0001, 0.00001 \}$,
    \item Momentum:  0.9
    \item Learning rate decay: 0.1 at epochs 20 and 40. 
    \item Batch Size $\{64, 128\}$,
    \item Adaptive algorithm learning rate: $\{1.0, 0.5, 0.1, *\}$, where (*) is the optimal learning rate,
    \item Mixing with uniform distribution: $\{0, 0.01, 0.1 \}$,
    \item Adaptive algorithm learning rate decay scheduling: $\{ \text{constant}, O(1/\sqrt{t}), \text{Adagrad} \} $,
    \item \softcvar algorithm temperature:  $\{0.1, 1, 10\}$.
    \item Random seeds:  $\{0, 1, 2, 3, 4\}$.
    \item Early stopping: $\{\text{True}, \text{False}\}$.
\end{enumerate}

\paragraph{Experimental Significance:} 
In UCI and synthetic datasets, the test results of each algorithm is paired because the same data split is used across the algorithms. Likewise, for vision datasets, the neural network initialization is paired across experiments. Therefore, we use a paired t-test to determine statistical significance for $p \leq 0.05$ \citep{zimmerman1997teacher}. 
In \trunkcvar and \softcvar experiments, we usually encountered numerical overflows, we discarded such experiments to determine significance. 

\paragraph{Evaluation Metric Normalization} \looseness -1
We evaluate the CVaR and the Average Loss for regression and we add Accuracy, and the CVaR and the Average Loss of the surrogate loss for classification. 
We normalize the mean score $s$ of algorithm $a$ on a given task by $(s_a - \min_a s_a) / (\max_a s_a - \min_a s_a)$ and we divide the standard deviation by $\max_a s_a$. 

\section{Extended Experimental Results} \label{app:DetailedExperiments}
In \Cref{tab:ConvexCvar} we show the results of \Cref{exp:convex}, in \Cref{tab:ShiftTest} the results of \Cref{exp:convex-dist}, and in \Cref{tab:VisionTest} the results of \Cref{exp:deep-cvar}.
\begin{table*}[htpb]
\scriptsize 
\centering
\setlength{\tabcolsep}{4.25pt}
\renewcommand{\arraystretch}{1.25}
\caption{
\looseness -1 Test mean $\pm$ s.d. over five independent data splits. In shaded bold we indicate the best algorithms.
In regression, we show the CVaR/(mean) loss. \adacvar is competitive to benchmarks optimizing the CVaR.
In classification, we show the accuracy / precision in ``Accuracy'' rows and CVaR/loss in ``Surrogate'' rows. 
\adacvar has similar accuracy to \mean/\softcvar, but with a lower CVaR. 
} \label{tab:ConvexCvar}
\begin{tabular}{cc|cc|cc|cc|cc}
    \toprule
    \multicolumn{2}{c|}{Data Set} &  \multicolumn{2}{c|}{\adacvar}  & \multicolumn{2}{c|}{\trunkcvar}  & \multicolumn{2}{c|}{\softcvar} & \multicolumn{2}{c}{\mean}  \\ \midrule
    \multirow{6}{*}{\STAB{\rotatebox[origin=c]{90}{Regression Loss}}\hspace{-3mm}} 
& Abalone & \winner{8.92 $\pm$ 3.3} & 0.61 $\pm$ 0.1 &  \winner{7.94 $\pm$ 2.7} & 0.79 $\pm$ 0.1 &  14.25 $\pm$ 0.3 & 0.63 $\pm$ 0.0 &  11.07 $\pm$ 3.7 & \winner{0.51 $\pm$ 0.1} \\
& Boston & \winner{3.09 $\pm$ 0.8} & 0.28 $\pm$ 0.1  &  \winner{3.02 $\pm$ 1.0} & 0.36 $\pm$ 0.0 &  3.28 $\pm$ 1.4 & \winner{0.24 $\pm$ 0.1} &  4.51 $\pm$ 1.9 & \winner{0.27 $\pm$ 0.1}\\
& Cpu & \winner{2.32 $\pm$ 0.2} & 0.58 $\pm$ 0.0  &  \winner{2.12 $\pm$ 0.2} & 0.57 $\pm$ 0.1  & 2.85 $\pm$ 0.0 & 0.40 $\pm$ 0.0 &  9.95 $\pm$ 1.4 & \winner{0.30 $\pm$ 0.0}\\
& Normal & \winner{0.22 $\pm$ 0.1} & 0.03 $\pm$ 0.0  &  0.42 $\pm$ 0.3 & 0.06 $\pm$ 0.0 &  \winner{0.18 $\pm$ 0.0} & \winner{0.01 $\pm$ 0.0} &  0.55 $\pm$ 0.2 & 0.07 $\pm$ 0.0\\
& Pareto & \winner{0.39 $\pm$ 0.3} & 0.02 $\pm$ 0.0  &  0.43 $\pm$ 0.1 & 0.05 $\pm$ 0.0 &  \winner{0.30 $\pm$ 0.3} & \winner{0.01 $\pm$ 0.0} &  0.69 $\pm$ 0.1 & 0.06 $\pm$ 0.0\\
& Sinc & \winner{7.70 $\pm$ 3.1} & \winner{0.80 $\pm$ 0.2}  &  \winner{7.82 $\pm$ 3.5} & \winner{0.74 $\pm$ 0.2} & \winner{7.82 $\pm$ 3.5} & \winner{0.72 $\pm$ 0.2} &  8.40 $\pm$ 3.9 & \winner{0.71 $\pm$ 0.2} \\
\midrule
\multirow{8}{*}{\STAB{\rotatebox[origin=c]{90}{{Classification Accuracy}}}\hspace{-3mm}} 
& Adult & \winner{0.85 $\pm$ 0.0} & 0.72 $\pm$ 0.0  &  0.71 $\pm$ 0.1 & 0.44 $\pm$ 0.2 &  \winner{0.85 $\pm$ 0.0} & \winner{0.74 $\pm$ 0.0}&  \winner{0.85 $\pm$ 0.0} & \winner{0.74 $\pm$ 0.0} \\
& Australian & \winner{0.82 $\pm$ 0.0} & \winner{0.79 $\pm$ 0.0}  &  0.66 $\pm$ 0.1 & 0.62 $\pm$ 0.1 &  \winner{0.82 $\pm$ 0.0} & \winner{0.81 $\pm$ 0.0} &  0.81 $\pm$ 0.0 & 0.78 $\pm$ 0.0 \\
& German & 0.74 $\pm$ 0.0 & 0.65 $\pm$ 0.0  &  0.64 $\pm$ 0.0 & 0.43 $\pm$ 0.0 &  \winner{0.78 $\pm$ 0.0} & \winner{0.71 $\pm$ 0.0} &  \winner{0.77 $\pm$ 0.0} & \winner{0.67 $\pm$ 0.1}\\
& Monks & \winner{0.62 $\pm$ 0.1} & \winner{0.59 $\pm$ 0.1}  &  0.53 $\pm$ 0.1 & 0.50 $\pm$ 0.1 &  \winner{0.63 $\pm$ 0.0} & 0.68 $\pm$ 0.0 &  0.61 $\pm$ 0.1 & 0.66 $\pm$ 0.1\\
& Phoneme & \winner{0.75 $\pm$ 0.0} & \winner{0.59 $\pm$ 0.0}  &  0.47 $\pm$ 0.1 & 0.26 $\pm$ 0.0 &  \winner{0.75 $\pm$ 0.0} & \winner{0.60 $\pm$ 0.0} &  \winner{0.76 $\pm$ 0.0} & \winner{0.60 $\pm$ 0.0} \\
& Spambase & 0.90 $\pm$ 0.0 & 0.88 $\pm$ 0.0  &  0.81 $\pm$ 0.0 & 0.71 $\pm$ 0.0 &  \winner{0.93 $\pm$ 0.0} & \winner{0.92 $\pm$ 0.0} &  \winner{0.92 $\pm$ 0.0} & 0.91 $\pm$ 0.0\\
& Splice & \winner{0.94 $\pm$ 0.0} & \winner{0.93 $\pm$ 0.0}  &  0.90 $\pm$ 0.0 & 0.89 $\pm$ 0.0 &  0.92 $\pm$ 0.0 & 0.91 $\pm$ 0.0 & 0.93 $\pm$ 0.0 & 0.92 $\pm$ 0.0\\
& Titanic & \winner{0.78 $\pm$ 0.0} & \winner{0.74 $\pm$ 0.0}  &  0.55 $\pm$ 0.2 & 0.40 $\pm$ 0.3 &  \winner{0.78 $\pm$ 0.0} & \winner{0.75 $\pm$ 0.0} &  \winner{0.78 $\pm$ 0.0} & \winner{0.75 $\pm$ 0.0}\\
\midrule
\multirow{8}{*}{\STAB{\rotatebox[origin=c]{90}{Classification Surrogate}} \hspace{-3mm}} 
& Adult & 1.95 $\pm$ 0.1 & 0.37 $\pm$ 0.0  &  \winner{0.70 $\pm$ 0.0} & 0.69 $\pm$ 0.0 &  3.08 $\pm$ 0.1 & \winner{0.32 $\pm$ 0.0} &  3.13 $\pm$ 0.1 & 0.32 $\pm$ 0.0\\
& Australian & 1.33 $\pm$ 0.1 & 0.49 $\pm$ 0.0  &  \winner{0.83 $\pm$ 0.0} & 0.66 $\pm$ 0.0 &  1.78 $\pm$ 0.1 & 0.47 $\pm$ 0.0 &  1.93 $\pm$ 0.1 & \winner{0.45 $\pm$ 0.0}\\
& German & 1.41 $\pm$ 0.2 & 0.56 $\pm$ 0.0  &  \winner{0.86 $\pm$ 0.0} & 0.67 $\pm$ 0.0 &  2.04 $\pm$ 0.2 & \winner{0.50 $\pm$ 0.0} &  1.90 $\pm$ 0.2 & \winner{0.51 $\pm$ 0.0}\\
& Monks & \winner{0.89 $\pm$ 0.0} & \winner{0.66 $\pm$ 0.0}  &  \winner{0.88 $\pm$ 0.0} & 0.68 $\pm$ 0.0 &  1.36 $\pm$ 0.1 & 0.69 $\pm$ 0.0 &  1.07 $\pm$ 0.0 & \winner{0.66 $\pm$ 0.0}\\
& Phoneme & 0.83 $\pm$ 0.0 & 0.64 $\pm$ 0.0  &  \winner{0.69 $\pm$ 0.0} & 0.69 $\pm$ 0.0 &  2.56 $\pm$ 0.0 & \winner{0.47 $\pm$ 0.0} &  2.63 $\pm$ 0.0 & \winner{0.47 $\pm$ 0.0} \\
& Spambase & \winner{1.04 $\pm$ 0.1} & 0.49 $\pm$ 0.0  &  1.15 $\pm$ 0.3 & 0.49 $\pm$ 0.0 &  6.30 $\pm$ 1.7 & \winner{0.23 $\pm$ 0.0} &  5.23 $\pm$ 1.4 & \winner{0.23 $\pm$ 0.0}\\
& Splice & 1.57 $\pm$ 0.2 & 0.27 $\pm$ 0.0  &  \winner{1.06 $\pm$ 0.1} & 0.49 $\pm$ 0.0 &  5.45 $\pm$ 0.9 & \winner{0.20 $\pm$ 0.0} &  1.75 $\pm$ 0.2 & \winner{0.22 $\pm$ 0.0}\\
& Titanic & 0.78 $\pm$ 0.0 & 0.66 $\pm$ 0.0 &  \winner{0.71 $\pm$ 0.0} & 0.70 $\pm$ 0.0 &  1.70 $\pm$ 0.0 & \winner{0.52 $\pm$ 0.0} &  1.68 $\pm$ 0.0 & \winner{0.52 $\pm$ 0.0} \\
\bottomrule
\end{tabular}
\end{table*}

\begin{table*}[htpb]
\scriptsize
\centering
\setlength{\tabcolsep}{5pt}
\renewcommand{\arraystretch}{1.25}
\caption{
Test accuracy/loss (mean $\pm$ s.d.) over five independent data splits with \emph{train/test distribution shift}. In shaded bold we indicate the best algorithms.
\adacvar has superior test accuracy than benchmarks. It has comparable test loss to \trunkcvar. 
} \label{tab:ShiftTest}
\begin{tabular}{cc|cc|cc|cc|cc}
    \toprule
    \multicolumn{2}{c|}{Data Set} &  \multicolumn{2}{c|}{\adacvar}  & \multicolumn{2}{c|}{\trunkcvar}  & \multicolumn{2}{c|}{\softcvar} & \multicolumn{2}{c}{\mean}  \\ \midrule
    \multirow{8}{*}{\STAB{\rotatebox[origin=c]{90}{Distribution Shift 10\%}}\hspace{-5mm}} 
& Adult      &  \winner{0.65 $\pm$ 0.0} & \winner{0.63 $\pm$ 0.0} &  0.61 $\pm$ 0.1 & 0.69 $\pm$ 0.0 &  0.63 $\pm$ 0.1 & 0.84 $\pm$ 0.3 &  0.60 $\pm$ 0.1 & 0.90 $\pm$ 0.4 \\
& Australian &  \winner{0.68 $\pm$ 0.0} & \winner{0.58 $\pm$ 0.0} &  0.46 $\pm$ 0.1 & 0.70 $\pm$ 0.0 &  0.48 $\pm$ 0.3 & 0.91 $\pm$ 0.4 &  0.46 $\pm$ 0.3 & 0.87 $\pm$ 0.3 \\
& German     &  \winner{0.48 $\pm$ 0.1} & \winner{0.73 $\pm$ 0.0} &  0.44 $\pm$ 0.1 & 0.75 $\pm$ 0.0 &  0.45 $\pm$ 0.1 & 1.00 $\pm$ 0.2 &  \winner{0.48 $\pm$ 0.1} & 0.80 $\pm$ 0.1 \\
& Monks      &  \winner{0.49 $\pm$ 0.0} & 0.72 $\pm$ 0.0 &  0.39 $\pm$ 0.1 & 0.73 $\pm$ 0.0 &  0.43 $\pm$ 0.3 & 1.27 $\pm$ 0.6 &  0.40 $\pm$ 0.2 & 0.91 $\pm$ 0.3 \\
& Phoneme    &  \winner{0.58 $\pm$ 0.1} & \winner{0.68 $\pm$ 0.0} &  0.52 $\pm$ 0.2 & 0.69 $\pm$ 0.0 &  0.51 $\pm$ 0.2 & 1.08 $\pm$ 0.4 &  0.55 $\pm$ 0.2 & 0.92 $\pm$ 0.3 \\
& Spambase   &  \winner{0.77 $\pm$ 0.0} & \winner{0.46 $\pm$ 0.0} &  0.74 $\pm$ 0.1 & 0.60 $\pm$ 0.0 &  0.71 $\pm$ 0.2 & 0.61 $\pm$ 0.2 &  0.69 $\pm$ 0.2 & 0.58 $\pm$ 0.2 \\
& Splice     &  \winner{0.84 $\pm$ 0.0} & \winner{0.40 $\pm$ 0.0} &  0.77 $\pm$ 0.1 & 0.57 $\pm$ 0.0 &  0.73 $\pm$ 0.1 & 0.52 $\pm$ 0.2 &  0.50 $\pm$ 0.3 & 0.81 $\pm$ 0.4 \\
& Titanic    &  0.46 $\pm$ 0.2 & \winner{0.70 $\pm$ 0.0} &  \winner{0.50 $\pm$ 0.3} & \winner{0.69 $\pm$ 0.0} &  0.43 $\pm$ 0.3 & 1.17 $\pm$ 0.5 &  0.41 $\pm$ 0.3 & 0.87 $\pm$ 0.3 \\
\bottomrule
\end{tabular}
\end{table*}

\begin{table*}[htpb]
\scriptsize 
\centering
\setlength{\tabcolsep}{4.25pt}
\renewcommand{\arraystretch}{1.25}
\caption{
\looseness -1 ``Accuracy'' rows: Avg.~test accuracy/(min class) precision.
``Surrogate'' rows: Avg.~test CVaR/loss.
Average over 5 different random seeds. 
\adacvar matches the accuracy and average loss performance of \mean, but has a lower CVaR. 
\trunkcvar and \softcvar struggle to learn in non-convex tasks.
} \label{tab:VisionTest}
\begin{tabular}{cc|cc|cc|cc|cc}
    \toprule
    \multicolumn{2}{c|}{Data Set} &  \multicolumn{2}{c|}{\adacvar}  & \multicolumn{2}{c|}{\trunkcvar}  & \multicolumn{2}{c|}{\softcvar} & \multicolumn{2}{c}{\mean}  \\ \midrule
\multirow{3}{*}{\STAB{\rotatebox[origin=c]{90}{{Accuracy}}}\hspace{-3mm}} 
& MNIST & \winner{0.99 $\pm$ 0.0} & \winner{0.98 $\pm$ 0.0} &  \winner{0.99 $\pm$ 0.0} & \winner{0.98 $\pm$ 0.0} &  \winner{0.99 $\pm$ 0.0} & \winner{0.98 $\pm$ 0.0} &  \winner{0.99 $\pm$ 0.0} & \winner{0.98 $\pm$ 0.0} \\
& Fashion-MNIST & \winner{0.99 $\pm$ 0.0} & \winner{0.99 $\pm$ 0.0} &  \winner{0.99 $\pm$ 0.0} & 0.98 $\pm$ 0.0 &  \winner{0.99 $\pm$ 0.0} & 0.98 $\pm$ 0.0 &  \winner{0.99 $\pm$ 0.0} & 0.97 $\pm$ 0.0 \\
& Cifar-10 & \winner{0.94 $\pm$ 0.0} & \winner{0.87 $\pm$ 0.0} &  0.59 $\pm$ 0.0 & 0.50 $\pm$ 0.0 &  0.86 $\pm$ 0.1 & 0.64 $\pm$ 0.2 &  \winner{0.94 $\pm$ 0.0} & \winner{0.87 $\pm$ 0.0} \\
\midrule
\multirow{3}{*}{\STAB{\rotatebox[origin=c]{90}{Surrogate}} \hspace{-3mm}} 
& MNIST & \winner{0.31 $\pm$ 0.0} & \winner{0.03 $\pm$ 0.0} &  0.35 $\pm$ 0.1 & 0.03 $\pm$ 0.0 &  0.41 $\pm$ 0.1 & 0.04 $\pm$ 0.0 &  0.33 $\pm$ 0.0 & \winner{0.03 $\pm$ 0.0} \\
& Fashion-MNIST & \winner{0.31 $\pm$ 0.0} & \winner{0.03 $\pm$ 0.0} &  0.67 $\pm$ 0.0 & 0.14 $\pm$ 0.0 &  0.35 $\pm$ 0.0 & 0.04 $\pm$ 0.0 &  0.38 $\pm$ 0.0 & 0.04 $\pm$ 0.0 \\
& Cifar-10 & 2.36 $\pm$ 0.1 & \winner{0.25 $\pm$ 0.1} &  \winner{2.02 $\pm$ 0.0} & 1.32 $\pm$ 0.0 &  2.79 $\pm$ 0.1 & 0.32 $\pm$ 0.0 &  2.49 $\pm$ 0.1 & \winner{0.24 $\pm$ 0.0} \\
\bottomrule
\end{tabular}
\end{table*}

\newpage 
\subsection{Comparison to other Techniques that Address Distribution Shift} \label{app:exp:dist-shift}

In \cref{exp:convex-dist}, how \adacvar and the other CVaR optimization algorithms guarded against distribution shift. Now, we compare how do this algorithms compare with algorithms that address distribution shift.
We compare against up-sampling and down-sampling techniques for distribution shift for {\em all} the benchmarks considered in the experiment section. 
Namely, we perform up-sampling coupled with \adacvar, \trunkcvar, \softcvar, and \mean. 
We expect up-sampling techniques to perform better than raw CVaR optimization when the distribution shift is random {\em and} the test distribution is {\em balanced}. 

The CVaR guards against {\em worst-case} distribution shifts for all distributions in the {\em DRO} set. 
Certainly, balanced test-sets are contained in the {\em DRO}, but the {\em worst-case} guarantees that CVaR brings might be too pessimistic for this particular case. 
We expect specialized techniques to perform better. 
Nonetheless, the privileged information that the test-set is balanced might not always be available. 
In such cases, it could be preferable to optimize for the CVaR of the up-sampled data set, possibly for a large value of $\alpha$.

\FloatBarrier
\paragraph{Train-Set Shift}
In this subsection, we repeat the distribution shift experiment. Here, the train set is shifted (1-to-10 ratio) and the test-set is kept constant. In sake of presentation clarity, we show again the results in \cref{table:train:}. 
We repeat the experiments but upsampling the training set to balance the dataset. 
We show the results in \cref{table:train:upsample}. 
When we compare \mean with upsampling (last column in \cref{table:train:upsample})  and \adacvar without upsampling (first column in \cref{table:test:}), we see that both methods are comparable. 
This is because \adacvar addresses this distribution shift algorithmically whereas \mean does so by re-balancing the dataset. 
Nevertheless, \adacvar can also re-sample the data set to balance it. This is \adacvar with upsampling (first column in \cref{table:train:upsample}, that outperforms all other algorithms with (and without) upsampling. 

\begin{table}[htpb]
\scriptsize
\centering
\setlength{\tabcolsep}{5pt}
\renewcommand{\arraystretch}{1.25}
\caption{
Test accuracy/loss (mean $\pm$ s.d. for 5 random seeds) for train-set distribution shift without train set re-balancing.
We highlight the best algorithms.
}
\label{table:train:}
\begin{tabular}{cc|cc|cc|cc|cc}
\toprule
\multicolumn{2}{c|}{Data Set} &  \multicolumn{2}{c|}{\adacvar}  & \multicolumn{2}{c|}{\trunkcvar}  & \multicolumn{2}{c|}{\softcvar} & \multicolumn{2}{c}{\mean}  \\
\midrule \multirow{8}{*}{\STAB{\rotatebox[origin=c]{90}{Train Set Shift}}\hspace{-5mm}}
& Adult      &  \winner{0.65 $\pm$ 0.0} & \winner{0.63 $\pm$ 0.0} &  0.61 $\pm$ 0.1 & 0.69 $\pm$ 0.0 &  0.63 $\pm$ 0.1 & 0.84 $\pm$ 0.3 &  0.60 $\pm$ 0.1 & 0.90 $\pm$ 0.4 \\
& Australian &  \winner{0.68 $\pm$ 0.0} & \winner{0.58 $\pm$ 0.0} &  0.46 $\pm$ 0.1 & 0.70 $\pm$ 0.0 &  0.48 $\pm$ 0.3 & 0.91 $\pm$ 0.4 &  0.46 $\pm$ 0.3 & 0.87 $\pm$ 0.3 \\
& German     &  \winner{0.48 $\pm$ 0.1} & \winner{0.73 $\pm$ 0.0} &  0.44 $\pm$ 0.1 & 0.75 $\pm$ 0.0 &  0.45 $\pm$ 0.1 & 1.00 $\pm$ 0.2 &  \winner{0.48 $\pm$ 0.1} & 0.80 $\pm$ 0.1 \\
& Monks      &  \winner{0.49 $\pm$ 0.0} & \winner{0.72 $\pm$ 0.0} &  0.39 $\pm$ 0.1 & 0.73 $\pm$ 0.0 &  0.43 $\pm$ 0.3 & 1.27 $\pm$ 0.6 &  0.40 $\pm$ 0.2 & 0.91 $\pm$ 0.3 \\
& Phoneme    &  \winner{0.58 $\pm$ 0.1} & \winner{0.68 $\pm$ 0.0} &  0.52 $\pm$ 0.2 & 0.69 $\pm$ 0.0 &  0.51 $\pm$ 0.2 & 1.08 $\pm$ 0.4 &  0.55 $\pm$ 0.2 & 0.92 $\pm$ 0.3 \\
& Spambase   &  \winner{0.77 $\pm$ 0.0} & \winner{0.46 $\pm$ 0.0} &  0.74 $\pm$ 0.1 & 0.60 $\pm$ 0.0 &  0.71 $\pm$ 0.2 & 0.61 $\pm$ 0.2 &  0.69 $\pm$ 0.2 & 0.58 $\pm$ 0.2 \\
& Splice     &  \winner{0.84 $\pm$ 0.0} & \winner{0.40 $\pm$ 0.0} &  0.77 $\pm$ 0.1 & 0.57 $\pm$ 0.0 &  0.73 $\pm$ 0.1 & 0.52 $\pm$ 0.2 &  0.50 $\pm$ 0.3 & 0.81 $\pm$ 0.4 \\
& Titanic    &  0.46 $\pm$ 0.2 & \winner{0.70 $\pm$ 0.0} &  \winner{0.50 $\pm$ 0.3} & \winner{0.69 $\pm$ 0.0} &  0.43 $\pm$ 0.3 & 1.17 $\pm$ 0.5 &  0.41 $\pm$ 0.3 & 0.87 $\pm$ 0.3 \\
\bottomrule
\end{tabular}
\end{table}

\begin{table}[htpb]
\scriptsize
\centering
\setlength{\tabcolsep}{5pt}
\renewcommand{\arraystretch}{1.25}
\caption{
Test accuracy/loss (mean $\pm$ s.d. for 5 random seeds) for train-set distribution shift with train set re-balancing via upsampling
We highlight the best algorithms.
}
\label{table:train:upsample}
\begin{tabular}{cc|cc|cc|cc|cc}
\toprule
\multicolumn{2}{c|}{Data Set} &  \multicolumn{2}{c|}{\adacvar}  & \multicolumn{2}{c|}{\trunkcvar}  & \multicolumn{2}{c|}{\softcvar} & \multicolumn{2}{c}{\mean}  \\
\midrule \multirow{8}{*}{\STAB{\rotatebox[origin=c]{90}{Train Set Shift \textbf{Upsample}}}\hspace{-5mm}}
& Adult      &  0.51 $\pm$ 0.0 & \winner{0.69 $\pm$ 0.0} &  0.62 $\pm$ 0.1 & \winner{0.69 $\pm$ 0.0} &  \winner{0.67 $\pm$ 0.1} & 0.73 $\pm$ 0.3 &  \winner{0.65 $\pm$ 0.2} & 0.78 $\pm$ 0.4 \\
& Australian &  \winner{0.77 $\pm$ 0.0} & \winner{0.56 $\pm$ 0.0} &  0.50 $\pm$ 0.1 & 0.69 $\pm$ 0.0 &  0.56 $\pm$ 0.3 & \winner{0.79 $\pm$ 0.4} &  0.54 $\pm$ 0.3 & \winner{0.77 $\pm$ 0.3} \\
& German     &  \winner{0.58 $\pm$ 0.0} & \winner{0.68 $\pm$ 0.0} &  0.46 $\pm$ 0.1 & 0.74 $\pm$ 0.0 &  0.50 $\pm$ 0.2 & 0.92 $\pm$ 0.2 &  0.51 $\pm$ 0.1 & 0.76 $\pm$ 0.1 \\
& Monks      &  \winner{0.63 $\pm$ 0.0} & \winner{0.66 $\pm$ 0.0} &  0.42 $\pm$ 0.1 & 0.72 $\pm$ 0.0 &  0.49 $\pm$ 0.3 & 1.10 $\pm$ 0.3 &  0.45 $\pm$ 0.2 & \winner{0.85 $\pm$ 0.3} \\
& Phoneme    &  \winner{0.65 $\pm$ 0.1} & \winner{0.68 $\pm$ 0.0} &  0.57 $\pm$ 0.2 & \winner{0.69 $\pm$ 0.0} &  0.57 $\pm$ 0.2 & 0.94 $\pm$ 0.4 &  0.60 $\pm$ 0.2 & \winner{0.82 $\pm$ 0.3} \\
& Spambase   &  \winner{0.88 $\pm$ 0.0} & \winner{0.43 $\pm$ 0.0} &  0.78 $\pm$ 0.1 & 0.60 $\pm$ 0.0 &  0.76 $\pm$ 0.2 & \winner{0.54 $\pm$ 0.2} &  0.74 $\pm$ 0.2 & \winner{0.51 $\pm$ 0.2} \\
& Splice     &  \winner{0.89 $\pm$ 0.0} & \winner{0.30 $\pm$ 0.0} &  0.80 $\pm$ 0.1 & 0.54 $\pm$ 0.1 &  0.77 $\pm$ 0.1 & 0.45 $\pm$ 0.2 &  0.61 $\pm$ 0.3 & 0.67 $\pm$ 0.4 \\
& Titanic    &  \winner{0.51 $\pm$ 0.2} & \winner{0.70 $\pm$ 0.0} &  \winner{0.54 $\pm$ 0.3} & \winner{0.69 $\pm$ 0.0} &  \winner{0.52 $\pm$ 0.3} & 1.02 $\pm$ 0.5 &  \winner{0.49 $\pm$ 0.3} & 0.80 $\pm$ 0.3 \\
\bottomrule
\end{tabular}
\end{table}
\FloatBarrier
\vspace{-1em}
\paragraph{Test Shift}
Next, we consider another distribution shift. Namely, the train set is kept constant and instead the the test-set is shifted (1-to-10 ratio).
In this case, up-sampling should not change the test results. 
We show the results without re-sampling in \cref{table:test:} and with upsampling in \cref{table:test:upsample}. In this setting, upsampling does not help because the training sets are balanced. Comparing each case, we see that \adacvar has superior performance to the benchmarks. 

\begin{table}[htpb]
\scriptsize
\centering
\setlength{\tabcolsep}{5pt}
\renewcommand{\arraystretch}{1.25}
\caption{
Test accuracy/loss (mean $\pm$ s.d. for 5 random seeds) for test-set distribution shift without train set re-balancing.
We highlight the best algorithms.
}
\label{table:test:}
\begin{tabular}{cc|cc|cc|cc|cc}
\toprule
\multicolumn{2}{c|}{Data Set} &  \multicolumn{2}{c|}{\adacvar}  & \multicolumn{2}{c|}{\trunkcvar}  & \multicolumn{2}{c|}{\softcvar} & \multicolumn{2}{c}{\mean}  \\
\midrule \multirow{8}{*}{\STAB{\rotatebox[origin=c]{90}{Test Set Shift}}\hspace{-5mm}}
& Adult      &  0.50 $\pm$ 0.0 & \winner{0.69 $\pm$ 0.0} &  0.64 $\pm$ 0.1 & \winner{0.69 $\pm$ 0.0} &  \winner{0.66 $\pm$ 0.1} & \winner{0.72 $\pm$ 0.1} &  0.64 $\pm$ 0.1 & 0.77 $\pm$ 0.1 \\
& Australian &  \winner{0.76 $\pm$ 0.0} & \winner{0.61 $\pm$ 0.0} &  0.53 $\pm$ 0.1 & 0.69 $\pm$ 0.0 &  0.61 $\pm$ 0.3 & 0.71 $\pm$ 0.2 &  0.58 $\pm$ 0.3 & 0.72 $\pm$ 0.1 \\
& German     &  \winner{0.52 $\pm$ 0.1} & \winner{0.69 $\pm$ 0.0} &  0.46 $\pm$ 0.1 & 0.73 $\pm$ 0.0 &  0.51 $\pm$ 0.1 & 0.90 $\pm$ 0.2 & \winner{0.52 $\pm$ 0.1} & 0.75 $\pm$ 0.1 \\
& Monks      &  \winner{0.63 $\pm$ 0.1} & \winner{0.66 $\pm$ 0.0} &  0.44 $\pm$ 0.2 & 0.71 $\pm$ 0.0 &  0.50 $\pm$ 0.2 & 0.99 $\pm$ 0.2 &  0.48 $\pm$ 0.2 & 0.81 $\pm$ 0.2 \\
& Phoneme    &  0.50 $\pm$ 0.2 & \winner{0.69 $\pm$ 0.0} &  0.54 $\pm$ 0.2 & \winner{0.69 $\pm$ 0.0} &  0.56 $\pm$ 0.2 & 0.92 $\pm$ 0.4 &  \winner{0.58 $\pm$ 0.1} & 0.82 $\pm$ 0.3 \\
& Spambase   &  \winner{0.91 $\pm$ 0.0} & \winner{0.52 $\pm$ 0.0} &  0.81 $\pm$ 0.1 & 0.62 $\pm$ 0.0 &  0.79 $\pm$ 0.1 & \winner{0.48 $\pm$ 0.2} &  0.77 $\pm$ 0.2 & \winner{0.46 $\pm$ 0.2} \\
& Splice     &  \winner{0.91 $\pm$ 0.0} & \winner{0.29 $\pm$ 0.0} &  0.82 $\pm$ 0.1 & 0.54 $\pm$ 0.1 &  0.81 $\pm$ 0.1 & \winner{0.39 $\pm$ 0.2} &  0.67 $\pm$ 0.3 & 0.58 $\pm$ 0.2 \\
& Titanic    &  \winner{0.47 $\pm$ 0.2} & \winner{0.70 $\pm$ 0.0} &  \winner{0.52 $\pm$ 0.3} & \winner{0.69 $\pm$ 0.0} &  \winner{0.52 $\pm$ 0.3} & 0.99 $\pm$ 0.2 &  \winner{0.50 $\pm$ 0.3} & 0.81 $\pm$ 0.2 \\
\bottomrule
\end{tabular}
\end{table}

\begin{table}[htpb]
\scriptsize
\centering
\setlength{\tabcolsep}{5pt}
\renewcommand{\arraystretch}{1.25}
\caption{
Test accuracy/loss (mean $\pm$ s.d. for 5 random seeds) for test-set distribution shift with train set re-balancing via upsampling.
We highlight the best algorithms.
}
\label{table:test:upsample}
\begin{tabular}{cc|cc|cc|cc|cc}
\toprule
\multicolumn{2}{c|}{Data Set} &  \multicolumn{2}{c|}{\adacvar}  & \multicolumn{2}{c|}{\trunkcvar}  & \multicolumn{2}{c|}{\softcvar} & \multicolumn{2}{c}{\mean}  \\
\midrule \multirow{8}{*}{\STAB{\rotatebox[origin=c]{90}{Test Set Shift \textbf{Upsample}}}\hspace{-5mm}}
& Adult      &  0.46 $\pm$ 0.0 & \winner{0.69 $\pm$ 0.0} &  \winner{0.65 $\pm$ 0.1} & \winner{0.69 $\pm$ 0.0} &  \winner{0.66 $\pm$ 0.1} & \winner{0.72 $\pm$ 0.1} &  0.64 $\pm$ 0.1 & 0.76 $\pm$ 0.1 \\
& Australian &  \winner{0.76 $\pm$ 0.0} & \winner{0.61 $\pm$ 0.0} &  0.55 $\pm$ 0.1 & 0.68 $\pm$ 0.0 &  0.64 $\pm$ 0.3 & 0.67 $\pm$ 0.4 &  0.61 $\pm$ 0.3 & 0.69 $\pm$ 0.3 \\
& German     &  \winner{0.52 $\pm$ 0.1} & \winner{0.69 $\pm$ 0.0} &  0.46 $\pm$ 0.1 & 0.73 $\pm$ 0.0 &  \winner{0.52 $\pm$ 0.1} & 0.89 $\pm$ 0.2 &  \winner{0.53 $\pm$ 0.1} & 0.75 $\pm$ 0.1 \\
& Monks      &  \winner{0.63 $\pm$ 0.1} & \winner{0.66 $\pm$ 0.0} &  0.46 $\pm$ 0.2 & 0.71 $\pm$ 0.0 &  0.51 $\pm$ 0.2 & 0.92 $\pm$ 0.5 &  0.49 $\pm$ 0.2 & 0.79 $\pm$ 0.2 \\
& Phoneme    &  0.50 $\pm$ 0.2 & \winner{0.69 $\pm$ 0.0} &  0.52 $\pm$ 0.2 & \winner{0.69 $\pm$ 0.0} &  0.55 $\pm$ 0.2 & 0.90 $\pm$ 0.4 &  \winner{0.57 $\pm$ 0.1} & 0.82 $\pm$ 0.3 \\
& Spambase   &  \winner{0.91} $\pm$ 0.0 & \winner{0.54 $\pm$ 0.0} &  0.82 $\pm$ 0.1 & 0.63 $\pm$ 0.0 &  0.80 $\pm$ 0.1 & \winner{0.44 $\pm$ 0.2} &  0.79 $\pm$ 0.2 & \winner{0.43 $\pm$ 0.2} \\
& Splice     &  \winner{0.92} $\pm$ 0.0 & \winner{0.28 $\pm$ 0.0} &  0.84 $\pm$ 0.1 & 0.53 $\pm$ 0.1 &  0.83 $\pm$ 0.1 & \winner{0.36 $\pm$ 0.2} &  0.71 $\pm$ 0.3 & 0.53 $\pm$ 0.4 \\
& Titanic    &  \winner{0.47 $\pm$ 0.2} & \winner{0.70 $\pm$ 0.0} &  \winner{0.52 $\pm$ 0.3} & \winner{0.69 $\pm$ 0.0} &  \winner{0.52 $\pm$ 0.3} & 0.99 $\pm$ 0.2 &  \winner{0.50 $\pm$ 0.3} & 0.81 $\pm$ 0.2 \\
\bottomrule
\end{tabular}
\end{table}

\FloatBarrier
\newpage 
\paragraph{Double Shift}
Finally, we consider a case where the train set is imbalanced (1-to-10 ratio), and the test set is even more imbalanced (1-to-100 ratio). 
We show the results without re-sampling in \cref{table:double:} and with upsampling in \cref{table:double:upsample}. In this case, we upsampling is detrimental. 
\adacvar clearly outperforms all the other algorithms with upsampling and without upsampling. 

\begin{table}[htpb]
\scriptsize
\centering
\setlength{\tabcolsep}{5pt}
\renewcommand{\arraystretch}{1.25}
\caption{
Test accuracy/loss (mean $\pm$ s.d. for 5 random seeds) for imbalanced train and test-sets without train set re-balancing.
We highlight the best algorithms.
}
\label{table:double:}
\begin{tabular}{cc|cc|cc|cc|cc}
\toprule
\multicolumn{2}{c|}{Data Set} &  \multicolumn{2}{c|}{\adacvar}  & \multicolumn{2}{c|}{\trunkcvar}  & \multicolumn{2}{c|}{\softcvar} & \multicolumn{2}{c}{\mean}  \\
\midrule \multirow{8}{*}{\STAB{\rotatebox[origin=c]{90}{Double Set Shift}}\hspace{-5mm}}
& Adult      &  \winner{0.94 $\pm$ 0.0} & \winner{0.50 $\pm$ 0.0} &  0.70 $\pm$ 0.2 & 0.69 $\pm$ 0.0 &  0.70 $\pm$ 0.2 & 0.63 $\pm$ 0.3 &  0.69 $\pm$ 0.2 & 0.66 $\pm$ 0.4 \\
& Australian &  \winner{0.90 $\pm$ 0.0} & \winner{0.39 $\pm$ 0.0} &  0.59 $\pm$ 0.2 & 0.67 $\pm$ 0.1 &  0.69 $\pm$ 0.3 & 0.59 $\pm$ 0.4 &  0.66 $\pm$ 0.3 & 0.61 $\pm$ 0.3 \\
& German     &  \winner{0.95 $\pm$ 0.0} & \winner{0.43 $\pm$ 0.0} &  0.52 $\pm$ 0.2 & 0.70 $\pm$ 0.1 &  0.58 $\pm$ 0.2 & 0.79 $\pm$ 0.3 &  0.59 $\pm$ 0.2 & 0.69 $\pm$ 0.2 \\
& Monks      &  \winner{0.95 $\pm$ 0.0} & \winner{0.48 $\pm$ 0.1} &  0.50 $\pm$ 0.2 & 0.70 $\pm$ 0.0 &  0.58 $\pm$ 0.3 & 0.81 $\pm$ 0.4 &  0.56 $\pm$ 0.3 & 0.72 $\pm$ 0.3 \\
& Phoneme    &  \winner{0.58 $\pm$ 0.3} & \winner{0.67 $\pm$ 0.0} &  \winner{0.57 $\pm$ 0.2} & 0.69 $\pm$ 0.0 &  \winner{0.61 $\pm$ 0.2} & 0.79 $\pm$ 0.4 &  \winner{0.63 $\pm$ 0.2} & 0.73 $\pm$ 0.3 \\
& Spambase   &  \winner{0.98 $\pm$ 0.0} & \winner{0.19 $\pm$ 0.0} &  0.85 $\pm$ 0.1 & 0.60 $\pm$ 0.1 &  0.83 $\pm$ 0.1 & 0.39 $\pm$ 0.2 &  0.82 $\pm$ 0.2 & 0.39 $\pm$ 0.2 \\
& Splice     &  \winner{0.98 $\pm$ 0.0} & \winner{0.18 $\pm$ 0.0} &  0.86 $\pm$ 0.1 & 0.52 $\pm$ 0.1 &  0.85 $\pm$ 0.1 & 0.31 $\pm$ 0.2 &  0.75 $\pm$ 0.3 & 0.46 $\pm$ 0.4 \\
& Titanic    &  \winner{0.85 $\pm$ 0.2} & \winner{0.49 $\pm$ 0.0} &  0.52 $\pm$ 0.3 & 0.69 $\pm$ 0.0 &  0.59 $\pm$ 0.3 & 0.85 $\pm$ 0.3 &  0.57 $\pm$ 0.3 & 0.74 $\pm$ 0.3 \\
\bottomrule
\end{tabular}
\end{table}

\begin{table}[htpb]
\scriptsize
\centering
\setlength{\tabcolsep}{5pt}
\renewcommand{\arraystretch}{1.25}
\caption{
Test accuracy/loss (mean $\pm$ s.d. for 5 random seeds) for imbalanced train and test-sets with train set re-balancing via up-sampling.
We highlight the best algorithms.}
\label{table:double:upsample}
\begin{tabular}{cc|cc|cc|cc|cc}
\toprule
\multicolumn{2}{c|}{Data Set} &  \multicolumn{2}{c|}{\adacvar}  & \multicolumn{2}{c|}{\trunkcvar}  & \multicolumn{2}{c|}{\softcvar} & \multicolumn{2}{c}{\mean}  \\
\midrule \multirow{8}{*}{\STAB{\rotatebox[origin=c]{90}{Double Set Shift \textbf{Upsample}}}\hspace{-5mm}}
& Adult      &  0.40 $\pm$ 0.1 & \winner{0.69 $\pm$ 0.0} &  \winner{0.68 $\pm$ 0.2} & \winner{0.69 $\pm$ 0.0} &  \winner{0.72 $\pm$ 0.2} & \winner{0.60 $\pm$ 0.3} &  \winner{0.71 $\pm$ 0.2} & \winner{0.62 $\pm$ 0.4} \\
& Australian &  \winner{0.86 $\pm$ 0.0} & \winner{0.46 $\pm$ 0.0} &  0.61 $\pm$ 0.2 & 0.66 $\pm$ 0.1 &  0.71 $\pm$ 0.3 & \winner{0.56 $\pm$ 0.2} &  0.67 $\pm$ 0.3 & 0.59 $\pm$ 0.2 \\
& German     &  \winner{0.78 $\pm$ 0.1} & \winner{0.58 $\pm$ 0.0} &  0.55 $\pm$ 0.2 & 0.69 $\pm$ 0.1 &  0.60 $\pm$ 0.2 & 0.75 $\pm$ 0.2 &  0.60 $\pm$ 0.2 & 0.68 $\pm$ 0.1 \\
& Monks      &  \winner{0.68 $\pm$ 0.1 }& \winner{0.64 $\pm$ 0.0} &  0.51 $\pm$ 0.2 & 0.69 $\pm$ 0.0 &  0.59 $\pm$ 0.2 & \winner{0.78 $\pm$ 0.3} &  0.57 $\pm$ 0.2 & \winner{0.72 $\pm$ 0.2} \\
& Phoneme    &  \winner{0.49 $\pm$ 0.3} & \winner{0.69 $\pm$ 0.0} &  \winner{0.59 $\pm$ 0.2} & \winner{0.69 $\pm$ 0.0} &  \winner{0.63 $\pm$ 0.2} & \winner{0.75 $\pm$ 0.2} &  \winner{0.65 $\pm$ 0.2} & \winner{0.70 $\pm$ 0.2} \\
& Spambase   &  \winner{0.95 $\pm$ 0.0} & \winner{0.26 $\pm$ 0.1} &  0.86 $\pm$ 0.1 & 0.60 $\pm$ 0.1 &  0.84 $\pm$ 0.1 & \winner{0.36 $\pm$ 0.2} &  0.83 $\pm$ 0.2 & \winner{0.37 $\pm$ 0.2} \\
& Splice     &  \winner{0.96 $\pm$ 0.0} & \winner{0.18 $\pm$ 0.0} &  0.87 $\pm$ 0.1 & 0.50 $\pm$ 0.1 & 0.86 $\pm$ 0.1 & \winner{0.29 $\pm$ 0.2} &  0.77 $\pm$ 0.3 & 0.43 $\pm$ 0.2 \\
& Titanic    &  \winner{0.62 $\pm$ 0.3} & \winner{0.68 $\pm$ 0.0} &  \winner{0.53 $\pm$ 0.3} & \winner{0.69 $\pm$ 0.0} &  \winner{0.58 $\pm$ 0.3} & 0.82 $\pm$ 0.4 &  \winner{0.57 $\pm$ 0.3} & \winner{0.72 $\pm$ 0.3} \\
\bottomrule
\end{tabular}
\end{table}

\FloatBarrier

\paragraph{Experiment Conclusion} 
Techniques that address class imbalance, such as up-sampling, are useful when there is a-priori knowledge that the test-set is balanced (cf. Train Shift experiment). 
When this is not the case, such techniques may be detrimental (cf. Double Shift experiment). 
\adacvar, and the CVaR DRO optimization problem, is orthogonal to techniques for imbalanced dataset and can be used together, as shown by the previous experiments. Overall, \adacvar also has lower standard errors than \mean and \softcvar. This indicates that it is more robust to the random sampling procedures. 
\newpage
\section{More Related Work}

\subsection{Combinatorial Bandit Algorithms}
A central contribution of our work is an \emph{efficient} sampling algorithm based on an instance of combinatorial bandits, the $k$-set problem.
In this setting, the learner must choose a subset of $k$ out of $N$ experts with maximum rewards, and there are $\binom{N}{k}$ such sets.
\citet{koolen2010hedging} introduce this problem in the full information setting ($k$-sets) and
\citet{cesa2012combinatorial} extend it to the bandit setting. \citet{cesa2012combinatorial} propose an algorithm, called CombBand, that attains a regret of $O(k^{3/2}\sqrt{NT \log(N/k)})$ when the learner receives bandit feedback.
\citet{audibert2013regret} prove a lower bound of $O(k \sqrt{NT})$, which is attained by \citet{alatur2020multi} up to a $\sqrt{\log(N)}$ factor.
However, the computational and space complexity of the CombBand algorithm is $O(kN^3)$ and $O(N^3)$, respectively.
\citet{uchiya2010algorithms} propose an efficient sampling algorithm that has $O(N\log(k))$ computational and $O(N)$ space complexity.
Instead, we efficiently represent the algorithm proposed by \citet{alatur2020multi} using Determinantal Point Processes \citep{kulesza2012determinantal}.
This yields an algorithm with $O(\log(N))$ computational and $O(N)$ space complexity.

\subsection{Sampling from \texorpdfstring{$k$}-DPPs}
We propose to directly sample from the marginals of the (relaxed) $k$-DPP. 
Our setting has the advantage that the k-DPP is diagonal and there is no need for performing an eigendecomposition of the kernel matrix, which takes $O(N^3)$ operations. 
However, there are efficient methods that avoid the eigendecomposition and return a sample (of size $k$) of the $k$-DPP. Sampling uniformly at random from such sample can be done in constant time.
State-of-the-art exact sampling methods for k-DPPs take at least $O(N \operatorname{poly}(k))$ operations
using rejection sampling \citep{derezinski2019exact}, whereas approximate methods that use MCMC have mixing times of $O(Nk)$ \citep{anari2016monte}.
Hence, such algorithms are ineficient compared to our sampling algorithm that takes only $O(\log(N))$. This is because our method is specialized on diagonal $k$-DPPs and latter algorithms remain valid for general $k$-DPPs.

\end{document}